\pdfoutput=1

\documentclass[11pt]{article}

\usepackage[preprint]{acl}

\usepackage{times}
\usepackage{latexsym}

\usepackage[T1]{fontenc}

\usepackage[utf8]{inputenc}

\usepackage{microtype}

\usepackage{inconsolata}

\usepackage{graphicx}

\usepackage{amsmath}
\usepackage{amsfonts}
\usepackage{amssymb}
\usepackage{multirow}
\usepackage{subfigure}

\usepackage{amsthm}
\newtheorem{theorem}{Theorem}

\usepackage{hyperref}
\hypersetup{
    colorlinks=true,
    linkcolor=cyan,
    filecolor=blue,      
    urlcolor=red,
    citecolor=green,
}
\usepackage{xcolor}
\usepackage{utfsym}
\usepackage{booktabs}

\title{KaLM: Knowledge-aligned Autoregressive Language Modeling via Dual-view Knowledge Graph Contrastive Learning}

\author{Peng Yu \textsuperscript{1},\ Cheng Deng\textsuperscript{1},\ Beiya Dai\textsuperscript{1},\ Xinbing Wang\textsuperscript{1},\ Ying Wen\textsuperscript{1}\thanks{\quad Ying Wen is the corresponding author.} \\
        \textsuperscript{1}Shanghai Jiao Tong University\\
        \texttt{ \{pursuit\_yp, davendw, beiya\_dai, xwang8, ying.wen\}@sjtu.edu.cn }\\
}

\begin{document}
\maketitle
\begin{abstract}
Autoregressive large language models (LLMs) pre-trained by next token prediction are inherently proficient in generative tasks. However, their performance on knowledge-driven tasks such as factual knowledge querying remains unsatisfactory. Knowledge graphs (KGs), as high-quality structured knowledge bases, can provide reliable knowledge for LLMs, potentially compensating for their knowledge deficiencies. 
Aligning LLMs with explicit, structured knowledge from KGs has been a challenge; previous attempts either failed to effectively align knowledge representations or compromised the generative capabilities of LLMs, leading to less-than-optimal outcomes. 
This paper proposes \textbf{KaLM}, a \textit{Knowledge-aligned Language Modeling} approach, which fine-tunes autoregressive LLMs to align with KG knowledge via the joint objective of explicit knowledge alignment and implicit knowledge alignment. The explicit knowledge alignment objective aims to directly optimize the knowledge representation of LLMs through dual-view knowledge graph contrastive learning. The implicit knowledge alignment objective focuses on incorporating textual patterns of knowledge into LLMs through triple completion language modeling. 
Notably, our method achieves a significant performance boost in evaluations of knowledge-driven tasks, specifically embedding-based knowledge graph completion and generation-based knowledge graph question answering. 
\end{abstract}

\section{Introduction}
\label{sec:introduction}
Large language models (LLMs) like PaLM 2 \cite{anil2023palm} and GPT-4 \cite{achiam2023gpt} have recently made remarkable advancements in a wide range of natural language processing tasks \cite{li2022pretrained,su2019generalizing}. 
However, LLMs still face challenges in tasks requiring factual or domain-specific knowledge, resulting in unsatisfactory performance in knowledge-driven tasks. 
From the perspective of knowledge representation, LLMs serve as parametric knowledge bases, providing implicit, non-deterministic knowledge, while knowledge graphs (KGs) function as structured knowledge bases, offering explicit, deterministic knowledge. 
KGs, commonly organized as factual knowledge triples describing relations between entities, can serve as a reliable knowledge source for LLMs. 
Aligning LLMs with KG knowledge can enhance the knowledge reasoning capabilities of LLMs and improve their performance on knowledge-driven tasks, such as knowledge graph completion (KGC) and knowledge graph question answering (KGQA).

Autoregressive LLMs pre-trained through next token prediction tasks often exhibit limitations in knowledge representation, leading to embeddings that lack diversity and specificity. This limitation becomes evident in tasks that demand distinctive sentence embeddings, such as dense retrieval and semantic search \cite{muennighoff2022sgpt,ma2023fine}. As demonstrated in Figure~\ref{fig-intro-a}, the representations generated by LLMs tend to be overly homogeneous across different pieces of knowledge, undermining their effectiveness in applications requiring fine-grained semantic distinctions.

\begin{figure*}[t]
\centering
\subfigure[LLaMA]{
\label{fig-intro-a}
\includegraphics[width=0.45\textwidth]{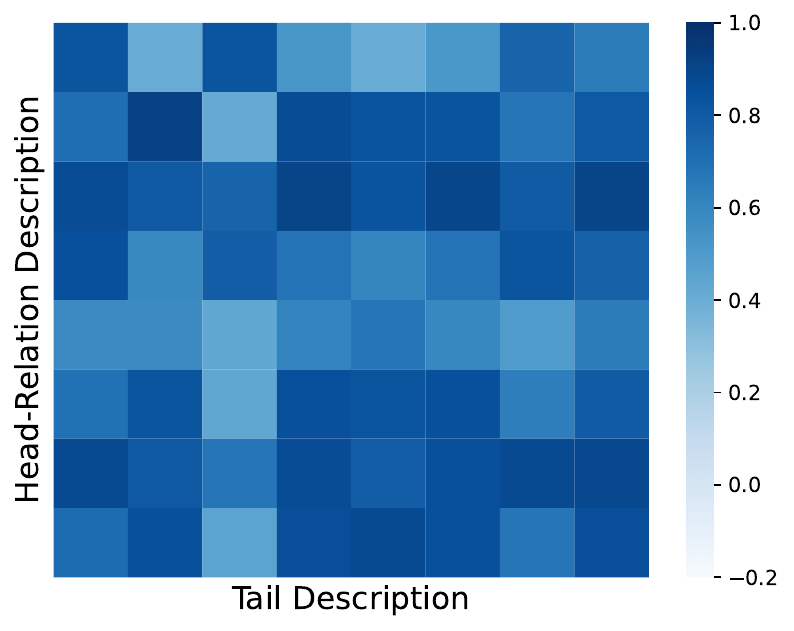}
}
\quad
\subfigure[KaLM]{
\label{fig-intro-b}
\includegraphics[width=0.45\textwidth]{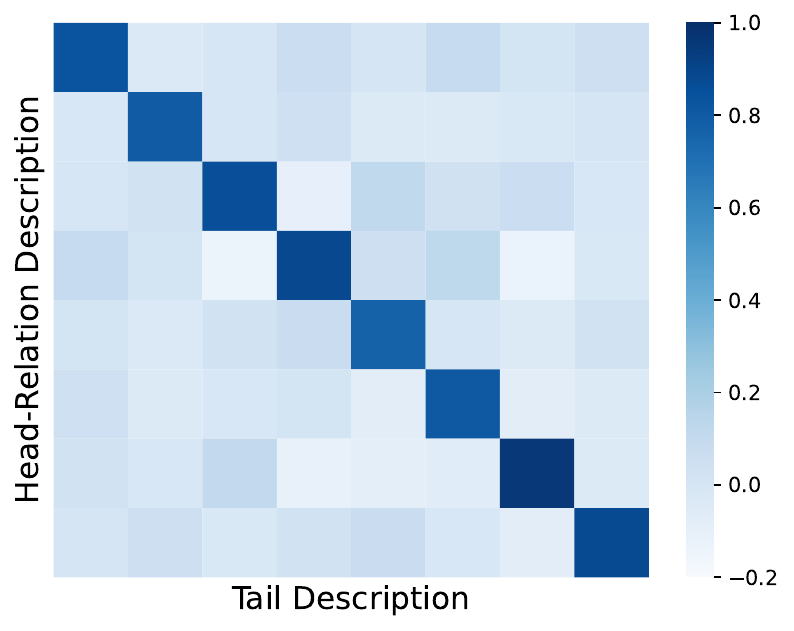}
}
\caption{Similarity matrix of knowledge representations of (a) Llama-2-7B \cite{touvron2023llama} and (b) KaLM. 
The values denote the cosine similarity between the head-relation and tail embedding. 
The diagonal elements represent positive <head-relation, tail> pairs from the same KG triple, 
which should maintain high similarity (darker color); 
off-diagonal elements represent negative <head-relation, tail> pairs from different KG triples, 
which should have lower similarity (lighter color).  
In an ideal setting, knowledge representations should be able to distinguish between different triples, while maintaining alignment and uniformity of the representation, as shown in Figure~\ref{fig-intro-b}. 
}
\label{fig-intro}
\end{figure*}


The concept of explicit knowledge alignment is introduced to directly optimize the knowledge representation within language models by devising direct knowledge training objectives. This strategy emerges in response to the observed degradation in knowledge representation within autoencoder-based pre-trained language models (PLMs), a phenomenon termed \textit{representation anisotropy} \cite{ethayarajh2019contextual}. This issue is characterized by the clustering of learned token and sentence embeddings within a constrained area of the representation space, leading to a lack of distributional uniformity \cite{li2020sentence}. While previous efforts to address representation anisotropy have largely concentrated on promoting uniformity among token representations, they often overlook the critical alignment of similar sentence representations \cite{su2021whitening,li2020sentence,su2022contrastive}. More recent works advocate for integrating KG triples and using knowledge graph embedding losses to fine-tune PLMs, aiming to bolster their knowledge representation abilities \cite{shen2022joint,wang2022language}. Nonetheless, such approaches may limit themselves to optimizing at the token level or reduce the model to a mere text encoder, thereby diminishing its inherent generative capabilities.


Conversely, implicit knowledge alignment leverages the pre-training or fine-tuning of language models with external knowledge sources, employing the vanilla language modeling objective or its variations. This approach predominantly preserves the next token prediction framework, essentially retaining the native text generation prowess of LLMs. In the realm of implicit knowledge alignment, the prevalent practice involves the fine-tuning of LLMs with KG triples and their textual descriptions, as opposed to directly altering the hidden knowledge representations \cite{chen2022knowledge,yao2023exploring}. Nevertheless, the efficacy of these methods on knowledge graph completion tasks remains substantially inferior when compared to strategies that directly fine-tune knowledge representations \cite{wang2022language,wang2022simkgc}. Intriguing findings from \cite{fu2023revisiting} reveal that fine-tuning PLMs with randomly unaligned KG triples can achieve performance on par with that obtained through fine-tuning with aligned triples in various tasks, including named entity recognition and relation classification. Their findings suggest that the hidden states of entities, whether infused with aligned or random knowledge, exhibit remarkable similarity. Consequently, existing implicit alignment methods fail to effectively utilize the injected knowledge or accurately discern the connection between newly introduced knowledge and the model's inherent knowledge, culminating in suboptimal performance.

In this paper, we propose \textbf{KaLM}, a \textit{Knowledge-aligned Language Modeling} approach for aligning LLMs with KG knowledge. Specifically, we use KG triples and their textual descriptions to fine-tune LLMs via the joint objective of \textit{explicit knowledge alignment} and \textit{implicit knowledge alignment.} 

The explicit knowledge alignment objective aims to directly optimize the hidden representations of knowledge in LLMs through \textit{dual-view knowledge graph contrastive learning}. 
We theoretically prove and empirically show that this objective can facilitate knowledge representation alignment and alleviate representation anisotropy. 
For KG triples, we consider tail entity description and the concatenation of head entity description and relation description as two distinct views of the same knowledge. 
\textit{The key insight is that: (1) representations of two different views of the same knowledge (i.e., from the same triple) should be pulled together, while (2) representations of different knowledge (i.e., from different triples) should be pushed apart.} 
The first term encourages semantically similar knowledge to remain close in the representation space, promoting knowledge representation alignment. 
The second term forces dissimilar knowledge to be as far apart as possible in the vector space, improving knowledge representation uniformity and mitigating representation anisotropy. 
As shown in Figure~\ref{fig-intro-b}, our method can obtain the ideal knowledge representations that are both aligned and uniform. 

The implicit knowledge alignment objective focuses on incorporating textual patterns of knowledge into LLMs through \textit{triple completion language modeling}, which can maintain the generative capability of LLMs and boost performance on knowledge inference tasks. 
We constructed a triple completion dataset based on the KG triples to fine-tune LLMs, improving their instruction-following ability and facilitating implicit knowledge alignment. 
We also show the implicit knowledge alignment objective can further boost knowledge representation performance. 
This confirms that both explicit alignment and implicit alignment are crucial for knowledge alignment, as they both essentially require a deep understanding of knowledge.

Our contributions are summarized as follows: 
\begin{itemize}
    \item We introduce \textbf{KaLM}, a \textit{knowledge-aligned language modeling} approach that aligns autoregressive LLMs with KG knowledge via the joint objective of \textit{explicit knowledge alignment} and \textit{implicit knowledge alignment.} 
    \item We \textit{theoretically prove and empirically demonstrate} that the explicit knowledge alignment objective achieved through dual-view knowledge graph contrastive learning can facilitate knowledge representation alignment and alleviate the issue of representation anisotropy. 
    \item The experimental results on knowledge-driven tasks demonstrate the effectiveness of \textit{KaLM}. In the embedding-based KGC task, KaLM significantly improves Mean Rank and Hit@10 metrics compared to previous state-of-the-art methods. In the generation-based KGQA task, KaLM achieves a notable improvement in answering accuracy compared to the base LLM. 
\end{itemize}

\section{Related Work}
\label{sec:related}
Our work is closely related to Knowledge Enhancement for LLMs and Representation Anisotropy of Language Models. 
A more detailed review of related work can be found in Appendix \ref{sec:appendix-related}.

\noindent \textbf{Knowledge Enhancement for LLMs} 
Knowledge enhancement aims to incorporate factual and domain-specific knowledge into LLMs to address their knowledge deficiencies. This can be divided into retrieval-based augmentation and training-based integration. 
\textit{Retrieval-based knowledge augmentation} methods leverage external retrieval modules to provide additional knowledge, aiming to improve the knowledge reasoning capability of LLMs \cite{sun2023think,jiang2023structgpt}. 
However, this approach may lead to knowledge conflicts \cite{feng2023trends}, where knowledge in LLMs and knowledge in the retrieved documents are inconsistent or the retrieved multiple documents are contradictory. 
\textit{Training-based knowledge integration} methods involve using KG triple descriptions to pre-train or fine-tune LLMs, aiming to achieve knowledge alignment. 
These methods can be divided into explicit alignment \cite{wang2021kepler,yasunaga2022deep} and implicit alignment \cite{yao2023exploring,zhang2023making} based on whether they directly optimize the knowledge representation. 
Nevertheless, prior methods have either sacrificed the generative capability or lacked effective representation alignment. 
Our approach enhances the knowledge of LLMs via a unique joint objective of explicit alignment and implicit alignment, improving the quality of knowledge representations and generative knowledge reasoning capabilities.

\noindent \textbf{Representation Anisotropy of Language Models} 
PLMs have long been plagued by representation anisotropy \cite{ethayarajh2019contextual}, where the learned token and sentence embeddings are confined to a narrow cone within the entire representation space. 
The issue of representation anisotropy not only results in model degradation \cite{su2022contrastive} but also leads to poor performance on discriminative tasks. 
Previous work on alleviating representation anisotropy has mainly focused on post-processing techniques such as normalizing flows \cite{li2020sentence} or whitening operations \cite{su2021whitening}. 
\citet{su2022contrastive} propose a contrastive training objective to encourage learning isotropic token representations. 
However, these methods mainly improve the isotropy of token representations without enhancing the discriminability of sentence representations. 
Our method improves the token-level and sentence-level representation anisotropy of LLMs through dual-view knowledge graph contrastive learning, and it has rigorous theoretical guarantees.

\section{Knowledge-aligned Autoregressive Language Modeling}
\label{sec:method}

In this section, we introduce \textbf{KaLM}, a \textit{Knowledge-aligned Language Modeling} approach for aligning LLMs with KG knowledge via the joint objective of \textit{explicit knowledge alignment} and \textit{implicit knowledge alignment}. The overview is shown in Figure~\ref{method-overview}.

\begin{figure*}[ht]
\centering 
\includegraphics[width=\textwidth]{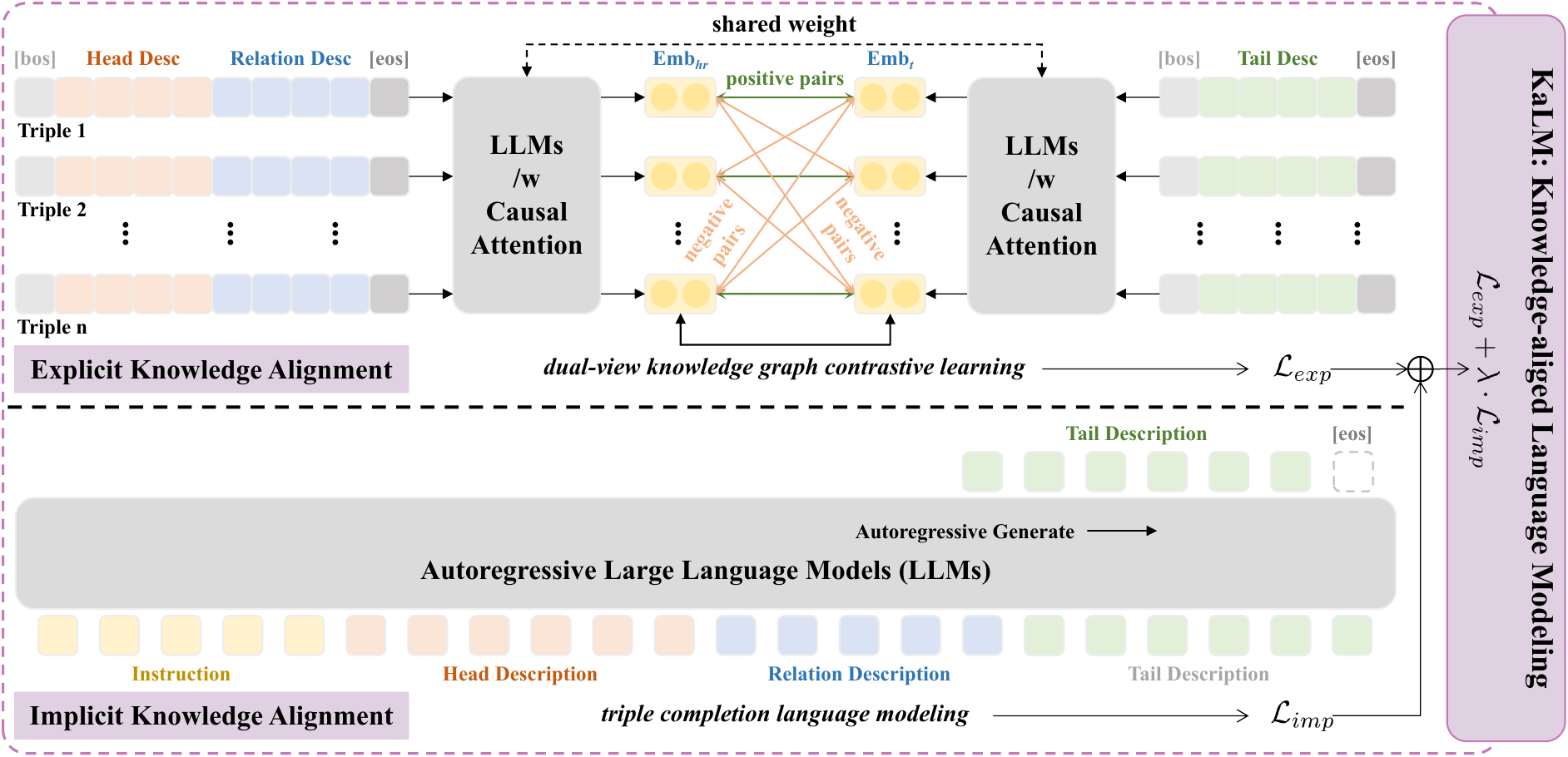}
\caption{The overall framework of \textbf{KaLM}. \textbf{Up}: The explicit knowledge alignment objective ($\mathcal{L}_{exp}$) aims to directly optimize the knowledge representation of LLMs via dual-view knowledge graph contrastive learning. \textbf{Down}: The implicit knowledge alignment objective ($\mathcal{L}_{imp}$) focuses on incorporating textual patterns of knowledge into LLMs via triple completion language modeling. The final training objective is the weighted average of $\mathcal{L}_{exp}$ and $\mathcal{L}_{imp}$.}
\label{method-overview}
\end{figure*} 

\subsection{Notations and Preliminaries}
A KG $\mathcal{G}$ stores factual knowledge, denoted as $\mathcal{G} = (\mathcal{E}, \mathcal{R}, \mathcal{T}, \mathcal{D})$. $\mathcal{E}$ and $\mathcal{R}$ are the set of entities and relations, respectively. $\mathcal{D}$ is the description set of all entities and relations. $\mathcal{D}_e$ and $\mathcal{D}_r$ are the textual description of entity $e$ and relation $r$, respectively. $\mathcal{T} = \{ (h, r, t) | h, t \in \mathcal{E}, r \in \mathcal{R} \}$ is the triple set. A triple $(h, r, t)$ depicts the fact that there is a relation $r$ between the head entity $h$ and the tail entity $t$.

\subsection{Explicit Knowledge Alignment}
For KG triples, the textual description of the tail entity and the concatenation of the textual descriptions of the head entity and relation can be seen as two distinct views of the same knowledge. 
This inspires \textit{KaLM} to align representations of two distinct views of the same knowledge (i.e., from the same triple), while separating representations of different knowledge (i.e., from different triples). 

The LLM, denoted as $E_{LLM}$, is fine-tuned with the \textit{dual-view knowledge graph contrastive learning} loss. The training corpus contains paired textual descriptions, $\{ (\mathcal{D}_{hr}, \mathcal{D}_t) \}_{i=1}^N$, where $\mathcal{D}_t$ is the tail entity description, and $\mathcal{D}_{hr}$ is the concatenation of the head entity description and relation description. 
Given a training pair $(\mathcal{D}_{hr}, \mathcal{D}_t)$, the same $E_{LLM}$ is used to compute the embeddings of $\mathcal{D}_{hr}$ and $\mathcal{D}_t$ independently. Moreover, we prepend the $\texttt{[bos]}$ token to the beginning and append the $\texttt{[eos]}$ token to the end of the textual description. 
The augmented input is fed into $E_{LLM}$, and the hidden representation corresponding to the $\texttt{[eos]}$ token from the last layer is used as the final embedding of the input. 
\begin{align*}
    e_{hr} = E_{LLM} (\texttt{[bos]}_{hr} \oplus \mathcal{D}_{hr} \oplus \texttt{[eos]}_{hr}), \\
    e_t = E_{LLM} (\texttt{[bos]}_t \oplus \mathcal{D}_t \oplus \texttt{[eos]}_t), 
\end{align*}
where $\oplus$ is the operation to concatenate two strings and $\mathcal{D}_{hr} = \mathcal{D}_h \oplus \mathcal{D}_r$. 
For stable training, we adopt ``['' as $\texttt{[bos]}_{hr}$ and ``]'' as $\texttt{[eos]}_{hr}$, while using ``\{'' as $\texttt{[bos]}_t$ and ``\}'' as $\texttt{[eos]}_t$. 

We utilize the knowledge graph contrastive learning loss to directly optimize the knowledge representation of the LLM by \textit{encouraging semantically similar knowledge to stay close in the representation space and pushing dissimilar knowledge to be far apart in the representation space}. More specifically, we apply the InfoNCE loss with an additive margin over the in-batch negatives to fine-tune the model. 
The row-direction loss $\ell_r$ is as follows for a given positive pair, and the column-direction loss $\ell_c$ is defined similarly (see Appendix~\ref{section-c2}). 
\begin{align}
    \label{eq-row-cl}
    \ell_r = - \log \frac{e^{(\phi(e_{hr}, e_t)-\gamma)/\tau}}{e^{(\phi(e_{hr}, e_t)-\gamma)/\tau} + \sum\nolimits_{i=1}^{\mathcal{N}} e^{\phi(e_{hr}, e_{t_i^{\prime}})/\tau}},
\end{align}
where $\mathcal{N}$ is the negative batch size, $\tau$ is the trainable temperature that controls the strength of penalties on hard negative samples, $\phi$ is the cosine similarity function that measures the plausibility of a triple, and $\gamma$ is the additive margin that encourages increasing the similarity score of positive pairs. 

The training objective for \textbf{exp}licit knowledge alignment is the sum of the $\ell_r$ and the $\ell_c$ losses: 
\begin{equation}
    \label{eq:l_rep}
    \mathcal{L}_{exp} = \frac{1}{\mathcal{N}} \underset{(\mathcal{ D}_{hr}, \mathcal{D}_t)}\sum (\ell_r + \ell_c) / 2.
\end{equation}

\subsection{Implicit Knowledge Alignment}
The implicit knowledge alignment objective focuses on incorporating textual patterns of knowledge into the LLM to prevent catastrophic forgetting of previous knowledge and maintain its generative capability. We constructed an instruction-tuning dataset based on the KG triple descriptions to fine-tune the model through \textit{triple completion language modeling}. 
We also show that the implicit knowledge alignment objective can bring performance boosts on knowledge representation evaluations. 
This indicates that explicit alignment and implicit alignment are both imperative for effective knowledge alignment, as they both essentially necessitate a profound understanding of knowledge. 

We follow the recipe of Stanford Alpaca \cite{alpaca} and use the provided template to construct the instruction-tuning dataset. The instruction passed to the template, abbreviated as $\texttt{inst}$, is: \textit{``Given the head entity and relation, write a tail entity that completes the triple''.} 
The input and output are $\mathcal{D}_{hr}$ and $\mathcal{D}_t$, respectively. 
The training objective for \textbf{imp}licit knowledge alignment is: 
\begin{equation}
    \label{eq:l_pat}
    \mathcal{L}_{imp} = \frac{1}{\mathcal{M}} \underset{(\mathcal{D}_{hr}, \mathcal{D}_t)}\sum - \log P(\mathcal{D}_t | \texttt{inst}, \mathcal{D}_{hr}),
\end{equation}
where $\mathcal{M}$ is the instruction-tuning batch size.

\subsection{Knowledge-aligned Language Modeling}
The ultimate training objective of our proposed \textbf{KaLM} is the weighted average of $\mathcal{L}_{exp}$ and $\mathcal{L}_{imp}$: 
\begin{equation}
\label{eq:final-obj}
    \mathcal{L}_{KaLM} = \mathcal{L}_{exp} + \lambda \cdot \mathcal{L}_{imp},
\end{equation}
where $\lambda$ is a hyperparameter that adjusts the relative weight between them. 
Notably, this formulation allows us to use different batch sizes for explicit knowledge alignment ($\mathcal{N}$) and implicit knowledge alignment ($\mathcal{M}$). Previous work has shown that a sufficiently large batch size is key to the success of contrastive representation learning \cite{chen2020simple}. With Equation~\ref{eq:final-obj}, we can significantly increase the explicit knowledge alignment batch size while keeping the implicit knowledge alignment batch size fixed to save computational resources.

\section{Theoretical Analysis}
\label{sec:theory}

We theoretically prove that the explicit knowledge alignment objective implemented through dual-view knowledge graph contrastive learning can facilitate knowledge representation alignment and alleviate the issue of representation anisotropy. 

\subsection{Dual-view Contrastive Learning for Knowledge Representation Alignment}
\label{sec:theory-1}
The outstanding performance of contrastive representation learning has attracted researchers to analyze its underlying reasons for success from a theoretical perspective. \citet{wang2020understanding} identify alignment and uniformity as two key properties of contrastive learning and propose two quantifiable metrics to measure the quality of representations. 

We concentrate on understanding the dual-view knowledge graph contrastive learning loss from the knowledge alignment and uniformity perspective. To simplify the notation, we use $f$ to denote $E_{LLM}$. 

\textit{Alignment} computes the expected distance between positive pairs and encourages the learned representations for positive pairs to be similar. 
\textit{Uniformity} evaluates the even distribution of representations and encourages the separation of features from randomly selected negative samples.

\begin{equation*}
    \ell_{\texttt{align}}(f; \alpha) \triangleq \underset{(\mathcal{D}_{hr}, \mathcal{D}_t)\sim p_{\texttt{pos}}}{\mathbb{E}} \left[ \Vert f(\mathcal{D}_{hr}) - f(\mathcal{D}_t) \Vert_2^{\alpha} \right] ,
\end{equation*}

\begin{equation*}
    \ell_{\texttt{uniform}}(f; t) \triangleq\log \underset{\mathcal{D}_i, \mathcal{D}_j\stackrel{i.i.d.}{\sim} p_{\texttt{data}}}{\mathbb{E}} \left[ e^{-t\Vert f(\mathcal{D}_i)-f(\mathcal{D}_j) \Vert_2^2} \right],
\end{equation*}
where $p_{\texttt{pos}}$ denotes the distribution of positive pairs $\{ (\mathcal{D}_{hr}, \mathcal{D}_t) \}_{i=1}^N$ and $p_{\texttt{data}}$ represents the data distribution of textual descriptions $\{\mathcal{D}_i\}_{i=1}^N$. 

Since the learned knowledge representations are L2-normalized, we have $\phi(e_{hr}, e_t) = f(x)^{\top} f(y)$. The additive margin $\gamma$ encourages the model to learn more robust features without affecting the asymptotic analysis, thus we ignore it. For ease of analysis, we reformulate the contrastive learning objective of Equation~\ref{eq-row-cl} and \ref{eq:l_rep} as follows: 
\begin{equation}
\begin{split}
    \label{eq:exp-reform}
    & \mathcal{L}_{\texttt{exp}}(f; \tau, \mathcal{N}) \triangleq 
    \underset{\substack{(\mathcal{D}_{hr}, \mathcal{D}_t) \sim p_{\texttt{pos}} \\ {\{\mathcal{D}_t{_i^{\prime}}\}_{i=1}^\mathcal{N}} \stackrel{i.i.d.}{\sim} p_{\texttt{data}}}} {\mathbb{E}} \\
    & \left[ - \log \frac{e^{f(\mathcal{D}_{hr})^{\top}f(\mathcal{D}_t)/\tau}}{e^{f(\mathcal{D}_{hr})^{\top}f(\mathcal{D}_t)/\tau} + \sum\limits_{i=1}^\mathcal{N} e^{f(\mathcal{D}_{hr})^{\top}f(\mathcal{D}_t{_i^{\prime}})/\tau}} \right],
\end{split}
\end{equation}


Following \citet{wang2020understanding}, we analyze the asymptotics of the objective in Equation~\ref{eq:exp-reform}.

\begin{theorem}[Asymptotics of $\mathcal{L}_{\texttt{exp}}$]
    \label{theorem-1}
    For temperature $\tau > 0$, as the number of negative samples $\mathcal{N} \rightarrow \infty$, the normalized dual-view knowledge graph contrastive loss in Equation~\ref{eq:exp-reform} converges to
    \begin{equation}
    \label{eq-asymptotics}
    \begin{split}
        & \lim_{\mathcal{N} \rightarrow \infty} \mathcal{L}_{\texttt{exp}}(f; \tau, \mathcal{N}) - \log \mathcal{N} = \\
        & \qquad -\frac{1}{\tau} \underset{(\mathcal{D}_{hr}, \mathcal{D}_t) \sim p_{\texttt{pos}}}{\mathbb{E}} \left[ f(\mathcal{D}_{hr})^{\top}f(\mathcal{D}_t) \right] \\
        & \qquad + \underset{\mathcal{D}_i \sim p_{data}}{\mathbb{E}}\left[\log\underset{\mathcal{D}_i^- \sim p_{\texttt{data}}}{\mathbb{E}}\left[e^{f(\mathcal{D}_i^-)^\top f(\mathcal{D}_i)/\tau}\right]\right].
    \end{split}
    \end{equation}
    We have the following conclusions:
    \begin{enumerate}
        \item By pulling together the representations of two different views of the same knowledge, the first term of Equation~\ref{eq-asymptotics} is minimized, and the encoder $E_{LLM}$ is perfectly knowledge-aligned. 
        \item Assuming the perfect uniform knowledge encoder $E_{LLM}$ exists, it precisely minimizes the second term of Equation~\ref{eq-asymptotics} by pushing away the representations of different knowledge. 
    \end{enumerate}
\end{theorem}
\begin{proof}
    See Appendix~\ref{sec:theory-1-proof}.
\end{proof}

\subsection{Alleviation of Representation Anisotropy}
\label{sec:theory-2}
We then prove that the dual-view knowledge graph contrastive learning objective can directly alleviate representation anisotropy and improve the discriminability of knowledge representations. 

Let $\mathbf{E}$ be the sentence embedding matrix of $\{\mathcal{D}_i\}_{i=1}^N$, where the $i$-th row of $\mathbf{E}$ is $e_i$. 
Following \citet{ethayarajh2019contextual}, the sentence-level representation anisotropy value of $\{\mathcal{D}_i\}_{i=1}^N$ is defined as: 
\begin{equation}
\label{eq-anisotropy}
    \texttt{anisotropy}_{\{\mathcal{D}\}} = \frac{1}{N(N-1)} \sum_{i=1}^N \sum_{j=1, j \neq i}^N e_i^\top e_j .
\end{equation}

We can further derive the following theorem. 

\begin{theorem}[Alleviation of Anisotropy]
\label{theorem-2}
    When $p_{data}$ is uniform over finite samples $\{\mathcal{D}_i\}_{i=1}^N$, the second term of Equation~\ref{eq-asymptotics} is the upper bound of the sentence-level anisotropy of $\{\mathcal{D}_i\}_{i=1}^N$, i.e.,
    \begin{equation}
    \label{eq-finite}
    \begin{split}
        & \underset{\mathcal{D}_i \sim p_{data}}{\mathbb{E}}\left[\log\underset{\mathcal{D}_i^-\sim p_{data}}{\mathbb{E}}\left[e^{f(\mathcal{D}_i^-)^\top f(\mathcal{D}_i)/\tau}\right]\right] \\ 
        & \quad \geq \frac{N-1}{\tau N} \cdot \texttt{anisotropy}_{\{\mathcal{D}\}} + \frac{1}{\tau N} .
    \end{split}
    \end{equation}
    We have the following result:
    By optimizing the second term of Equation~\ref{eq-asymptotics}, we essentially minimize the upper bound of the sentence-level anisotropy of corpus $\{\mathcal{D}_i\}_{i=1}^N$, thereby directly alleviating the representation anisotropy problem. 
\end{theorem}
\begin{proof}
    See Appendix~\ref{sec:theory-2-proof}.
\end{proof}

\section{Experiments}
\label{sec:experiments}
In this section, we assess the effectiveness of KaLM in knowledge alignment. 
The experimental setup is outlined in \ref{exp:setup}. 
In \ref{exp:kgc} and \ref{exp:kgqa}, we present results on knowledge graph completion (KGC) and knowledge graph question answering (KGQA). 
In \ref{exp:case}, we provide further analysis of knowledge representation and present case studies of KGQA generations.

\begin{table*}[t]
\centering
\caption{Embedding-based KGC results on WN18RR and FB15k-237. Baseline results are from their papers, with ``-'' indicating a missing result. The best and second-best results are marked by \textbf{bold} and \underline{underline}, respectively.}
{\begin{tabular}{l|ccccc|ccccc}
\hline
\multirow{2}{*}{\bf Method} & \multicolumn{5}{c|}{\bf WN18RR}      & \multicolumn{5}{c}{\bf FB15k-237}  \\ \cline{2-11} & \bf MR & \bf MRR & \bf H@1 & \bf H@3 & \bf H@10 & \bf MR & \bf MRR & \bf H@1 & \bf H@3 & \bf H@10 \\ \hline
\multicolumn{11}{l}{\textit{structure-based methods}}         \\ \hline
TransE & 2300 & 0.243 & 0.043  & 0.441 & 0.532 & 323 & 0.279 & 0.198 & 0.376 & 0.441 \\
DistMult & 7000 & 0.444 & 0.412 & 0.470 & 0.504  & 512 &  0.281 & 0.199 & 0.301 & 0.446 \\
RotatE & 3340 & 0.476 & 0.428 & 0.492 & 0.571 & 177 & 0.338 & 0.241 & 0.375 & 0.533 \\ \hline
\multicolumn{11}{l}{\textit{description-based methods (autoencoder PLMs)}}         \\ \hline
StAR & 51 & 0.401 & 0.243 & 0.491 & 0.709 & 117 & 0.296 & 0.205 & 0.322 & 0.482 \\ 
C-LMKE & 72 & 0.598 & 0.480 & 0.675 & 0.806 & 183 & \textbf{0.404} & \textbf{0.324} & \textbf{0.439} & \textbf{0.556} \\
SimKGC & - & \textbf{0.671} & \textbf{0.587} & \textbf{0.731} & 0.817 & - & \underline{0.333} & \underline{0.246} & \underline{0.362} & 0.510 \\ \hline
\multicolumn{11}{l}{\textit{description-based methods (autoregressive LLMs)}}         \\ \hline
Llama-2-7B  & 15969 & 0.010 & 0.004 & 0.010 &0.020 & 5359 & 0.006 & 0.002 & 0.004  & 0.012 \\
\textbf{Llama2-7B$_{KaLM}$}  & \textbf{19} & 0.556 & 0.409 & 0.656 & 0.851 & \textbf{114} & 0.299 & 0.204 & 0.325  & 0.502 \\ 
\textbf{Llama3-8B$_{KaLM}$}  & 23 & 0.588 & 0.446 & 0.676 & \underline{0.860} & 121 & 0.308 & 0.212 & 0.337  & 0.509 \\ 
\textbf{Mistral-7B$_{KaLM}$}  & \underline{20} & \underline{0.612} & \underline{0.484} & \underline{0.702} & \textbf{0.869} & \underline{116} & 0.317 & 0.225 & 0.351  & \underline{0.518} \\
\hline
\end{tabular}}
\label{kgc-results}
\end{table*}

\subsection{Experimental Setup}
\label{exp:setup}
\textbf{Datasets.} 
We use WN18RR \cite{dettmers2018convolutional} and FB15k-237 \cite{toutanova2015observed} as the KGs for knowledge alignment training. WN18RR and FB15k-237 are derived from WordNet and Freebase, respectively \cite{bordes2013translating}. We use the information provided by KG-BERT \cite{yao2019kg} for textual descriptions. Following \citet{wang2022simkgc}, we add an inverse triple $(t, r^{-1}, h)$ for each triple $(h, r, t)$ in the triple set, where $r^{-1}$ is the inverse relation of the original relation $r$. 

\noindent \textbf{Model Training.} 
We choose Llama-2-7B, Llama-3-8B, and Mistral-7B as base LLMs and fine-tune them through the joint objective of explicit knowledge alignment and implicit knowledge alignment. 
To save computational resources for parameter-efficient fine-tuning, we use LoRA \cite{hu2021lora} to fine-tune the feed-forward network of the model. 

\noindent \textbf{Evaluation Details.} 
Experiments mainly focus on two aspects: knowledge representation assessment and knowledge inference evaluation. 
For \textit{knowledge representation assessment}, we evaluate the embedding-based KGC task and illustrate the alleviation of representation anisotropy. We report five automated metrics: Mean Rank (MR), Mean Reciprocal Rank (MRR), and Hit@$k$ ($k \in \{1, 3, 10\}$). We compare KaLM with structure- and description-based methods. Structured-based methods include TransE \cite{bordes2013translating}, DistMult \cite{yang2015embedding}, and RotatE \cite{sun2018rotate}. Description-based methods include StAR \cite{wang2021structure}, C-LMKE \cite{wang2022language}, and SimKGC \cite{wang2022simkgc}. 
For \textit{knowledge inference evaluation}, we evaluate the generation-based KGQA task and analyze the PPL metric and MMLU score \cite{hendrycks2020measuring}. 
We report the prediction accuracy over entities, relations, and triples. 
We also provide case studies of KGQA generations. 

Additional experimental results and detailed ablation studies can be found in Appendix~\ref{sec:appendix-results} and \ref{sec:ablations}.

\begin{figure}[t]
\centering 
\includegraphics[width=\linewidth]{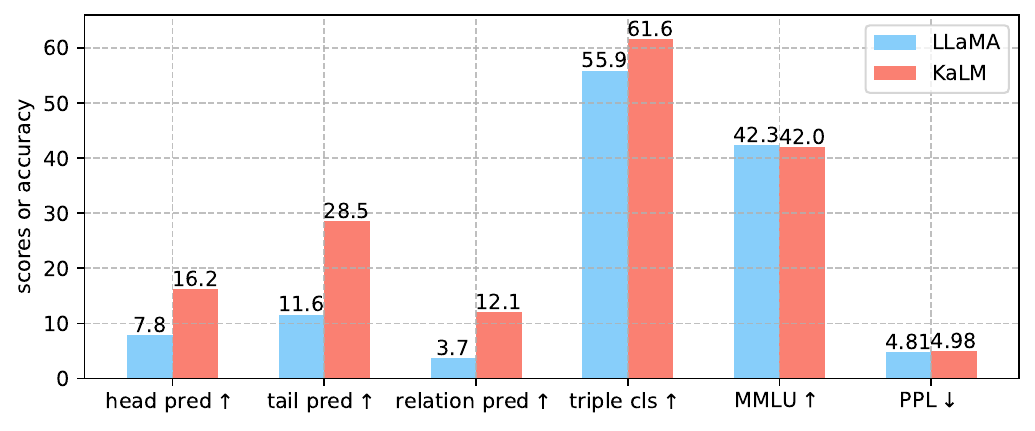}
\caption{Comparison of generative knowledge inference performance between Llama-2-7B and KaLM. $\uparrow$ means higher is better and $\downarrow$ means lower is better.}
\label{kgqa-results}
\end{figure}

\begin{figure}[t]
\centering
\subfigure[LLaMA]{
    \begin{minipage}[b]{0.45\linewidth}
    \includegraphics[width=1\linewidth]{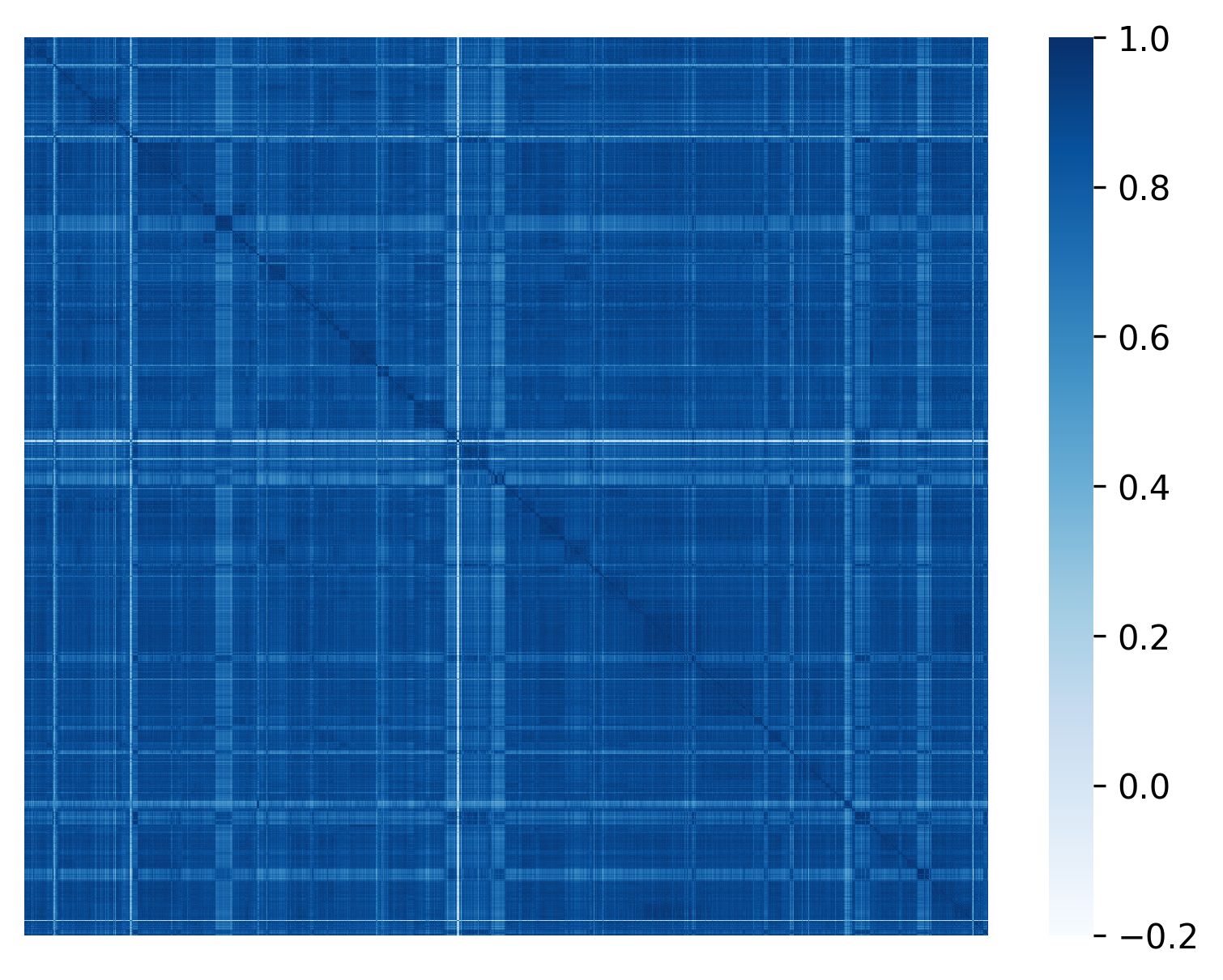}
    \end{minipage}
    \label{fig-anisotropy-exp-base}
}
\subfigure[KaLM]{
    \begin{minipage}[b]{0.45\linewidth}
    \includegraphics[width=1\linewidth]{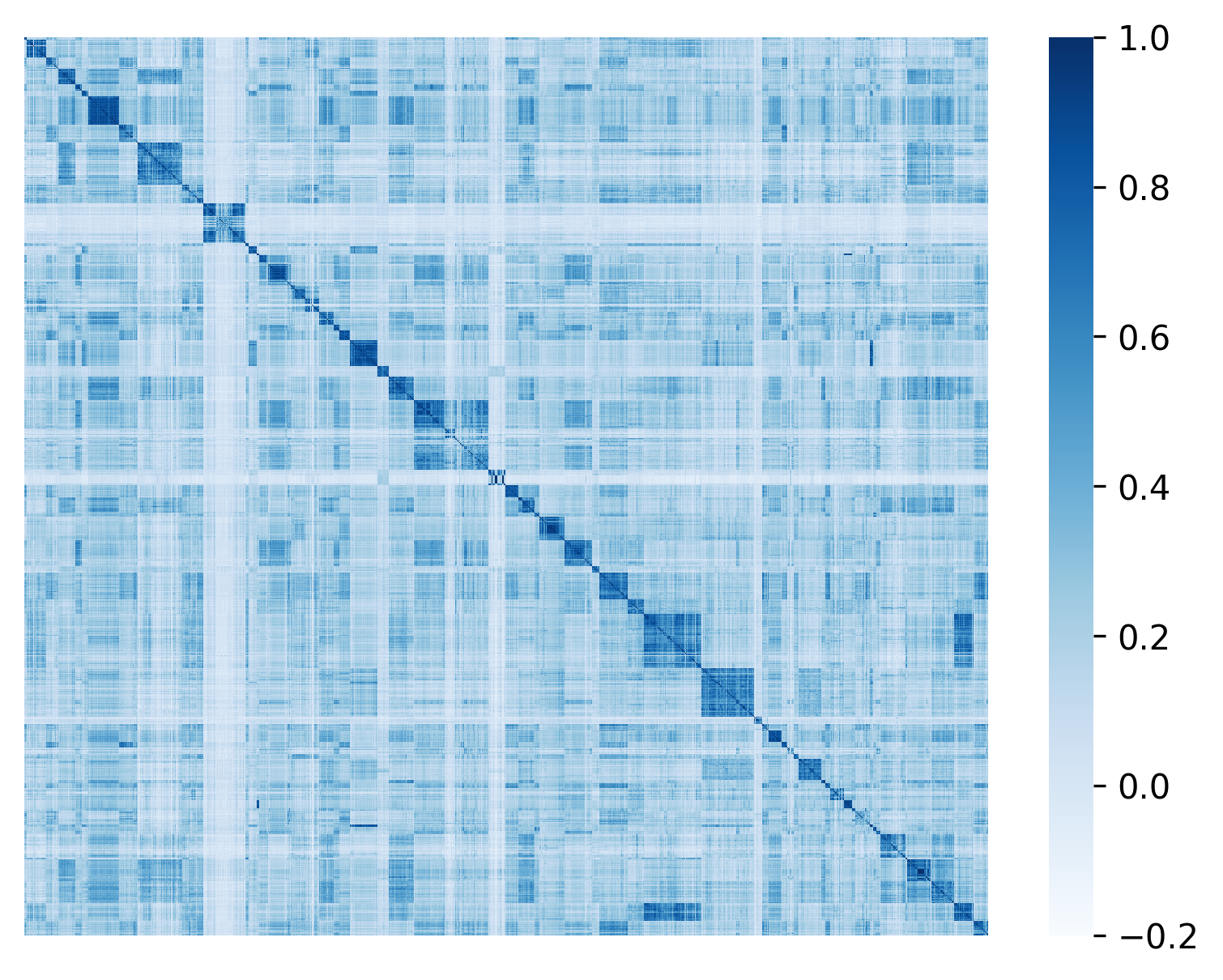}
    \end{minipage}
    \label{fig-anisotropy-exp-kalm}
}
\caption{Similarity matrix on the Wikitext-103 test set. From top-left to bottom-right, element $(i, j)$ denotes the cosine similarity between the $i$-th and the $j$-th sentence.}
\label{fig-anisotropy-exp}
\end{figure}

\begin{figure*}[t]
\centering 
\includegraphics[width=\textwidth]{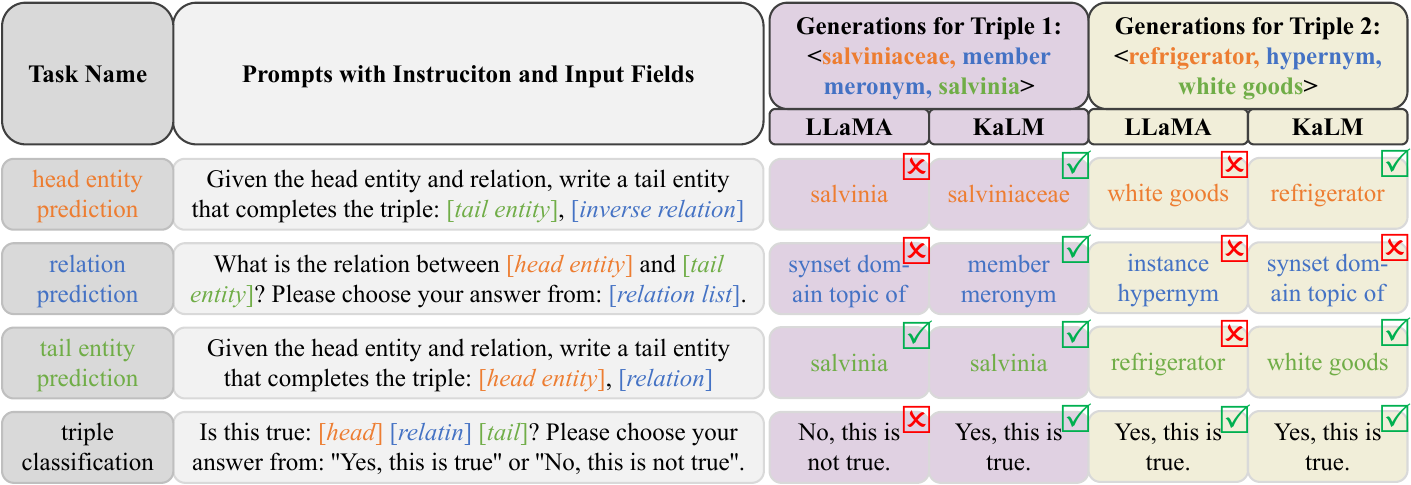}
\caption{Case studies of Llama-2-7B and KaLM on KGQA tasks. Note that the \textcolor[RGB]{237,125,49}{head entity}, \textcolor[RGB]{68,114,196}{relation}, and \textcolor[RGB]{122,173,71}{tail entity} are denoted by different colors. The \textcolor[RGB]{0,176,80}{$\protect\usym{1F5F9}$} mark indicates the correct answer, while \textcolor[RGB]{255,0,0}{$\protect\usym{1F5F5}$} signifies an incorrect answer.}
\label{kgqa-cases}
\end{figure*}

\subsection{Knowledge Representation Assessment}
\label{exp:kgc}
The embedding-based KGC results are shown in Table~\ref{kgc-results}. 
The base LLM failed to finish this task, with all metrics lagging far behind. 
On the WN18RR dataset, our method surpasses prior methods by a substantial margin in terms of MR and Hit@10. Other metrics fall slightly short of state-of-the-art methods, yet remain competitive. 
The performance of KaLM on FB15k-237 is slightly inferior, but it still achieves the best MR. 
Previous description-based methods generally perform poorly on FB15k-237, possibly due to the absence of effective textual descriptions. 
An example relation description from FB15k-237 is ``\textit{/music/artist/origin}'', which is quite vague and abstract. 
SimKGC uses a large batch size through intricate negative sampling methods and incorporates neighbor description augmentation and neighbor-based re-ranking techniques. 
C-LMKE uses self-adversarial negative sampling and utilizes extra entity degree information. 
These tricks enable SimKGC and C-LMKE to achieve higher performance. 
\textit{Using a larger batch size and more techniques can further improve other metrics of KaLM.} 
Overall, the results reveal that KaLM notably enhances the quality of knowledge representation, bringing performance boosts in KGC tasks.

\subsection{Knowledge Inference Evaluation}
\label{exp:kgqa}
The generation-based KGQA results are depicted in Figure~\ref{kgqa-results}. 
Llama-2-7B performs poorly in entity prediction and relation prediction. 
Our method demonstrates a significant performance boost in all generation-based KGQA tasks, including head/tail entity prediction, relation prediction, and triple classification. 
Furthermore, despite a slight increase in perplexity (PPL) scores on Wikitext-103 \cite{merity2016pointer} test set, our method still shows competitive performance in the MMLU test. 
The results demonstrate that KaLM achieves effective knowledge alignment, bringing in significantly improved KGQA performance while preserving the original generative and knowledge inference capabilities.

\subsection{Visualization of Knowledge Representation and Case Studies}
\label{exp:case}
\textbf{We provide visualization results to illustrate knowledge representation improvements.} 
Figure~\ref{fig-anisotropy-exp} shows the sentence similarity matrix of Llama-2-7B and KaLM on Wikitext-103. 
The diagonal elements denote the similarity of the same sentence, so the values are always 1. 
From color intensity, it is evident that KaLM learns more discriminative sentence representations, while Llama-2-7B assigns high similarity for arbitrary sentences. 
The sentences are organized by celebrities and their careers, thus there should also be a high similarity between adjacent sentences. 
This phenomenon is reflected in the similarity matrix of KaLM in Figure~\ref{fig-anisotropy-exp-kalm}, manifested in the smaller matrices with darker colors along the diagonal. 
\textit{More concretely, numerical analysis shows that after training with our method, the sentence-level anisotropy value significantly decreased from 0.83 to 0.21.}

\noindent \textbf{We present KGQA generation cases to demonstrate knowledge inference enhancements.} 
Figure~\ref{kgqa-cases} illustrates concrete examples of KGQA generation results on the WN18RR dataset. 
We showcase the responses generated by Llama-2-7B and KaLM for four tasks involving \textcolor[RGB]{237,125,49}{head entity prediction}, \textcolor[RGB]{68,114,196}{relation prediction}, \textcolor[RGB]{122,173,71}{tail entity prediction}, and triple classification. 
The prompt templates for each subtask are shown in the second column of Figure~\ref{kgqa-cases}, where the ``\textit{inverse relation}'' is the original relation description with a prefix word ``\textit{inverse}'' and the ``\textit{relation list}'' consists of all relations concatenated by the symbol ``|''. 
We display the generated answers for triple \textit{<\textcolor[RGB]{237,125,49}{salviniaceae}, \textcolor[RGB]{68,114,196}{member meronym}, \textcolor[RGB]{122,173,71}{salvinia}>} and triple <\textit{\textcolor[RGB]{237,125,49}{refrigerator}, \textcolor[RGB]{68,114,196}{hypernym}, \textcolor[RGB]{122,173,71}{white goods}}>. 
The base LLaMA frequently gives wrong answers and tends to identify keywords from the input prompts for prediction. 
In contrast, our method can understand the questions and correctly answer various KGQA tasks in most cases.

\section{Conclusion}
\label{sec:conclusion}

In this work, we show that the subpar performance of LLMs on knowledge-driven tasks stems from a lack of effective knowledge alignment. 
We present \textbf{KaLM}, a novel knowledge-aligned language modeling approach for aligning autoregressive LLMs with KG knowledge. 
Specifically, we identify two imperative objectives to achieve knowledge alignment: \textit{explicit knowledge alignment} and \textit{implicit knowledge alignment}. 
We conducted comprehensive experiments and analyses on embedding-based KGC and generation-based KGQA. 
Experimental results demonstrate that our method achieves effective knowledge alignment and consistently improves performance on knowledge-driven tasks.

\newpage
\section*{Limitations}
\label{sec:limitations}
There are several future directions to improve this work. 
Firstly, due to the limitation of computational resources, we used the limited-scale LLMs to train and evaluate our method. Evaluations on larger-scale LLMs, such as the 13B and 70B models, can further validate the effectiveness of our approach. 
Secondly, we use a simple linear combination of explicit alignment loss and implicit alignment loss as the final training objective for KaLM. Further investigations into various forms of loss combinations remain to be explored to maximize the utility of knowledge-aligned language modeling. 
Finally, we can delve into the performance of the knowledge representations obtained from knowledge-aligned language modeling in cross-domain applications such as retrieval-augmented generation, to gain broader insights into the generalization capabilities of the proposed approach.



\bibliography{acl_latex}

\newpage
\appendix

\section{More Detailed Review of Related Work}
\label{sec:appendix-related}

This work focuses on fine-tuning autoregressive LLMs to align with KG knowledge. 
Our work intersects with the following research areas: Knowledge Enhancement for LLMs, Knowledge Graph Completion, Contrastive Representation Learning, and Representation Anisotropy of Language Models.

\subsection{Knowledge Enhancement for LLMs}
Knowledge enhancement aims to incorporate factual and domain-specific knowledge into LLMs to address their knowledge deficiencies. This can be divided into retrieval-based knowledge augmentation and training-based knowledge integration. 
\textit{Retrieval-based knowledge augmentation} methods leverage external retrieval modules to provide additional knowledge, aiming to improve the knowledge reasoning capability of LLMs \cite{sun2023think,jiang2023structgpt}. 
However, this approach may lead to knowledge conflicts \cite{feng2023trends}, where the knowledge in LLMs and the knowledge in the retrieved documents are inconsistent or the retrieved multiple documents are contradictory. 
\textit{Training-based knowledge integration} methods involve using the textual descriptions of KG triples to pre-train or fine-tune LLMs, aiming to achieve knowledge alignment. 
These methods can be categorized into explicit alignment \cite{wang2021kepler,yasunaga2022deep} and implicit alignment \cite{yao2023exploring,zhang2023making} based on whether they directly optimize the knowledge representation. 
Nevertheless, these methods have either sacrificed the generative capability or lacked effective representation alignment. 
Our approach enhances the knowledge of LLMs via a unique joint objective of explicit alignment and implicit alignment, improving the quality of knowledge representations and generative knowledge reasoning capabilities.

\subsection{Knowledge Graph Completion}
Knowledge graph completion (KGC) refers to inferring missing triples from an incomplete KG, which can be used to evaluate the knowledge reasoning ability and knowledge representation quality of LLMs. 
Existing KGC methods can be categorized into structure-based and description-based. 
\textit{Structure-based methods} represent entities and relations as fixed-dimensional vector embeddings and use scoring functions to assess the plausibility of triples \cite{bordes2013translating,sun2019rotate}. 
\textit{Description-based methods} further incorporate the textual descriptions of KG triples and leverage pre-trained language models to learn knowledge representations of entities and relations \cite{yao2019kg,shen2022joint,wang2022language}. 
However, structure-based methods fail to generalize to unseen entities and relations, while description-based methods lack interpretability and exhibit lower efficiency when dealing with extremely large KGs.

\subsection{Contrastive Representation Learning}
Contrastive learning has demonstrated remarkable success in learning representations across various domains \cite{chen2020simple,liu2021hit,gunel2020supervised}. 
The goal is to learn representations that capture shared information between positive pairs while remaining invariant to perturbing noise. 
The commonly used contrastive learning objectives share a standardized design involving a softmax function over cosine similarity of paired features, with a temperature parameter to control the penalty strength on hard negative samples. 
\citet{wang2020understanding} propose understanding contrastive learning through the lens of alignment and uniformity on the hypersphere. 
\citet{wang2021understanding} show that temperature in the contrastive loss controls the strength of penalties over negative samples.

\subsection{Representation Anisotropy of Language Models}
PLMs have long been plagued by \textit{representation anisotropy} \cite{ethayarajh2019contextual}, where the learned token and sentence representations are confined to a narrow cone within the entire representation space. 
The issue of representation anisotropy not only results in model degradation \cite{su2022contrastive} but also leads to poor performance on discriminative tasks \cite{muennighoff2022sgpt}. 
Previous work on alleviating representation anisotropy has mainly focused on post-processing techniques such as normalizing flows \cite{li2020sentence} or whitening operations \cite{su2021whitening} to obtain isotropic representations. 
\citet{su2022contrastive} propose a contrastive training objective to encourage learning isotropic token representations. 
However, these methods mainly improve the isotropy of token representations without enhancing the discriminability of sentence representations. 
Our method improves the token-level and sentence-level representation anisotropy of LLMs through dual-view knowledge graph contrastive learning, and it has rigorous theoretical guarantees.

\section{Proofs for Theoretical Analysis}
\label{sec:appendix-theory}

In this section, we present proofs for theorems in Sections \ref{sec:theory-1} and \ref{sec:theory-2} of the main paper. 

\subsection{Proof of Theorem \ref{theorem-1} in Section \ref{sec:theory-1}}
\label{sec:theory-1-proof}

Recall the reformulated dual-view knowledge graph contrastive learning objective (Equation~\ref{eq:exp-reform}): 
\begin{equation*}
\begin{split}
    \label{eq:exp-reform-1}
    & \mathcal{L}_{\texttt{exp}}(f; \tau, \mathcal{N}) \triangleq 
    \underset{\substack{(\mathcal{D}_{hr}, \mathcal{D}_t) \sim p_{\texttt{pos}} \\ {\{\mathcal{D}_t{_i^{\prime}}\}_{i=1}^\mathcal{N}} \stackrel{i.i.d.}{\sim} p_{\texttt{data}}}} {\mathbb{E}} \\
    & \left[ - \log \frac{e^{f(\mathcal{D}_{hr})^{\top}f(\mathcal{D}_t)/\tau}}{e^{f(\mathcal{D}_{hr})^{\top}f(\mathcal{D}_t)/\tau} + \sum\limits_{i=1}^\mathcal{N} e^{f(\mathcal{D}_{hr})^{\top}f(\mathcal{D}_t{_i^{\prime}})/\tau}} \right] .
\end{split}
\end{equation*}
From the symmetry of $p$, we can derive: 
\begin{equation*}
\begin{split}
    \label{eq:exp-reform-2}
    & \mathcal{L}_{\texttt{exp}}(f; \tau, \mathcal{N}) = \\
    & \underset{(\mathcal{D}_{hr}, \mathcal{D}_t) \sim p_{\texttt{pos}}}{\mathbb{E}} \left[ -f(\mathcal{D}_{hr})^{\top}f(\mathcal{D}_t)/\tau \right] 
    + \underset{\substack{(\mathcal{D}_{hr}, \mathcal{D}_t) \sim p_{\texttt{pos}} \\ {\{\mathcal{D}_t{_i^{\prime}}\}_{i=1}^\mathcal{N}} \stackrel{i.i.d.}{\sim} p_{\texttt{data}}}} {\mathbb{E}} \\
    & \left[ \log \left( e^{f(\mathcal{D}_{hr})^{\top}f(\mathcal{D}_t)/\tau} + \sum\limits_{i=1}^\mathcal{N} e^{f(\mathcal{D}_t{_i^{\prime}})^{\top}f(\mathcal{D}_t)/\tau} \right) \right] .
\end{split}
\end{equation*}
Note that we can have the following limits almost surely by the strong law of large numbers (SLLN):
\begin{equation*}
\begin{split}
    \label{eq:exp-reform-3}
    & \lim_{\mathcal{N} \rightarrow \infty} \log 
    \left( \frac{e^{f(\mathcal{D}_{hr})^{\top}f(\mathcal{D}_t)/\tau}}{\mathcal{N}} 
    + \frac{\sum\limits_{i=1}^\mathcal{N} e^{f(\mathcal{D}_t{_i^{\prime}})^{\top}f(\mathcal{D}_t)/\tau}}{\mathcal{N}} \right) \\
    & = \log \underset{\mathcal{D}_i^- \sim p_{\texttt{data}}}{\mathbb{E}} 
    f(\mathcal{D}_i^-)^\top f(\mathcal{D}_i)/\tau .
\end{split}
\end{equation*}
Then we can derive the following limits:
\begin{equation*}
    \begin{split}
        & \lim_{\mathcal{N} \rightarrow \infty} \mathcal{L}_{\texttt{exp}}(f; \tau, \mathcal{N}) - \log \mathcal{N} \\
        & = \underset{(\mathcal{D}_{hr}, \mathcal{D}_t) \sim p_{\texttt{pos}}}{\mathbb{E}} \left[ -f(\mathcal{D}_{hr})^{\top}f(\mathcal{D}_t)/\tau \right] \\
        & \quad + \lim_{\mathcal{N} \rightarrow \infty} 
        \underset{\substack{(\mathcal{D}_{hr}, \mathcal{D}_t) \sim p_{\texttt{pos}} \\ {\{\mathcal{D}_t{_i^{\prime}}\}_{i=1}^\mathcal{N}} \stackrel{i.i.d.}{\sim} p_{\texttt{data}}}} {\mathbb{E}} \\
        & \quad \left[ \log \left( \frac{e^{f(\mathcal{D}_{hr})^{\top}f(\mathcal{D}_t)/\tau}}{\mathcal{N}} 
        + \frac{\sum\limits_{i=1}^\mathcal{N} e^{f(\mathcal{D}_t{_i^{\prime}})^{\top}f(\mathcal{D}_t)/\tau}}{\mathcal{N}} \right) \right] \\
        & = \underset{(\mathcal{D}_{hr}, \mathcal{D}_t) \sim p_{\texttt{pos}}}{\mathbb{E}} \left[ -f(\mathcal{D}_{hr})^{\top}f(\mathcal{D}_t)/\tau \right] 
    \end{split}
\end{equation*}
\begin{equation*}
    \begin{split}
        & \quad + \mathbb{E} \left[ \lim_{\mathcal{N} \rightarrow \infty} \log \left( \frac{e^{f(\mathcal{D}_{hr})^{\top}f(\mathcal{D}_t)/\tau}}{\mathcal{N}} 
        + \frac{\sum\limits_{i=1}^\mathcal{N} e^{f(\mathcal{D}_t{_i^{\prime}})^{\top}f(\mathcal{D}_t)/\tau}}{\mathcal{N}} \right) \right] \\
        & = -\frac{1}{\tau} \underset{(\mathcal{D}_{hr}, \mathcal{D}_t) \sim p_{\texttt{pos}}}{\mathbb{E}} \left[ f(\mathcal{D}_{hr})^{\top}f(\mathcal{D}_t) \right] \\
        & \quad + \underset{\mathcal{D}_i \sim p_{data}}{\mathbb{E}}\left[\log\underset{\mathcal{D}_i^- \sim p_{\texttt{data}}}{\mathbb{E}}\left[e^{f(\mathcal{D}_i^-)^\top f(\mathcal{D}_i)/\tau}\right]\right].
    \end{split}
\end{equation*}
We now finish the \textit{proof of Theorem \ref{theorem-1}}. 
\begin{equation*}
    \label{eq-asymptotics-yyy}
    \begin{split}
        & \lim_{\mathcal{N} \rightarrow \infty} \mathcal{L}_{\texttt{exp}}(f; \tau, \mathcal{N}) - \log \mathcal{N} = \\
        & \quad -\frac{1}{\tau} \underset{(\mathcal{D}_{hr}, \mathcal{D}_t) \sim p_{\texttt{pos}}}{\mathbb{E}} \left[ f(\mathcal{D}_{hr})^{\top}f(\mathcal{D}_t) \right] \\
        & \quad + \underset{\mathcal{D}_i \sim p_{data}}{\mathbb{E}}\left[\log\underset{\mathcal{D}_i^- \sim p_{\texttt{data}}}{\mathbb{E}}\left[e^{f(\mathcal{D}_i^-)^\top f(\mathcal{D}_i)/\tau}\right]\right].
    \end{split}
\end{equation*}

\subsection{Proof of Theorem \ref{theorem-2}  in Section \ref{sec:theory-2}}
\label{sec:theory-2-proof}

Recall the asymptotics of the explicit knowledge alignment objective when the number of negative samples approaches infinity (Equation~\ref{eq-asymptotics}): 
\begin{equation*}
\label{eq-asymptotics-1}
\begin{split}
    & \lim_{\mathcal{N} \rightarrow \infty} \mathcal{L}_{\texttt{exp}}(f; \tau, \mathcal{N}) - \log \mathcal{N} = \\
    & \qquad -\frac{1}{\tau} \underset{(\mathcal{D}_{hr}, \mathcal{D}_t) \sim p_{\texttt{pos}}}{\mathbb{E}} \left[ f(\mathcal{D}_{hr})^{\top}f(\mathcal{D}_t) \right] \\
    & \qquad + \underset{\mathcal{D}_i \sim p_{data}}{\mathbb{E}}\left[\log\underset{\mathcal{D}_i^- \sim p_{\texttt{data}}}{\mathbb{E}}\left[e^{f(\mathcal{D}_i^-)^\top f(\mathcal{D}_i)/\tau}\right]\right] .
\end{split}
\end{equation*}
Recall the definition of sentence-level anisotropy value of corpus $\{\mathcal{D}_i\}_{i=1}^N$ (Equation~\ref{eq-anisotropy}): 
\begin{equation*}
\label{eq-anisotropy-1}
    \texttt{anisotropy}_{\{\mathcal{D}\}} = \frac{1}{N(N-1)} \sum_{i=1}^N \sum_{j=1, j \neq i}^N e_i^\top e_j .
\end{equation*}
We can further derive the inequality below from the second term of Equation~\ref{eq-asymptotics} with Jensen's inequality when $p_{\texttt{data}}$ is uniform over finite samples $\{\mathcal{D}_i\}_{i=1}^N$: 

\begin{equation*}
\label{eq-finite-1}
\begin{split}
    & \underset{\mathcal{D}_i \sim p_{data}}{\mathbb{E}}\left[\log\underset{\mathcal{D}_i^-\sim p_{data}}{\mathbb{E}}\left[e^{f(\mathcal{D}_i^-)^\top f(\mathcal{D}_i)/\tau}\right]\right] \\ 
    & = \frac{1}{N}\sum_{i=1}^N\log\left(\frac{1}{N}\sum_{j=1}^N e^{e_i^\top e_j/\tau}\right) \\
    & \geq \frac{1}{\tau N^2}\sum_{i=1}^N\sum_{j=1}^N e_i^\top e_j \\
    & = \frac{1}{\tau N^2} \left( \sum_{i=1}^N \sum_{j=1, j \neq i}^N e_i^\top e_j + N \right) \\ 
    & = \frac{N-1}{\tau N} \cdot \frac{1}{N(N-1)} \sum_{i=1}^N \sum_{j=1, j \neq i}^N e_i^\top e_j + \frac{1}{\tau N} \\
    & = \frac{N-1}{\tau N} \cdot \texttt{anisotropy}_{\{\mathcal{D}\}} + \frac{1}{\tau N} .
\end{split}
\end{equation*}
We now finish the \textit{proof of Theorem \ref{theorem-2}}. 
\begin{equation*}
    \label{eq-finite-2}
    \begin{split}
        & \underset{\mathcal{D}_i \sim p_{data}}{\mathbb{E}}\left[\log\underset{\mathcal{D}_i^-\sim p_{data}}{\mathbb{E}}\left[e^{f(\mathcal{D}_i^-)^\top f(\mathcal{D}_i)/\tau}\right]\right] \\ 
        & \quad \geq \frac{N-1}{\tau N} \cdot \texttt{anisotropy}_{\{\mathcal{D}\}} + \frac{1}{\tau N} .
    \end{split}
\end{equation*}

\section{Further Details about Implementation and Experimental Setup}
\label{sec:appendix-setup}

\subsection{Dataset Details}

WN18RR and FB15k-237 are commonly used KGs derived from WordNet and Freebase, respectively \cite{bordes2013translating}. They have been carefully constructed to prevent test set leakage by removing inverse relations. We use these datasets for training and evaluation. 
The statistics are shown in Table~\ref{kgc-statistics}. 

\begin{table}[ht]
\centering
\caption{Statistics of the datasets.}
\scalebox{0.75}{\begin{tabular}{@{}l|lllll@{}}
\toprule
\textbf{Dataset}          & \textbf{\#Entity} & \textbf{\#Relation} & \textbf{\#Train} & \textbf{\#Valid} & \textbf{\#Test} \\ \midrule
WN18RR           &   $40,943$  &    $11$    &  $86,835$   &   $3,034$   &  $3,134$     \\
FB15k-237        &   $14,541$  &    $237$  &  $272,115$  &  $17,535$ & $20,466$ \\
\bottomrule
\end{tabular}}
\label{kgc-statistics}
\end{table}

\subsection{KaLM Implementation Details}
\label{section-c2}

We initially choose Llama-2-7B as the base LLM and fine-tune it through the training objective in Equation~\ref{eq:final-obj}. 
We use varying batch sizes for explicit knowledge alignment and implicit knowledge alignment. For WN18RR, we use a batch size of 24 for explicit alignment and 4 for implicit alignment. For FB15k-237, the batch sizes are 40 for explicit alignment and 6 for implicit alignment.
To save computing resources for parameter-efficient fine-tuning, we use the LoRA \cite{hu2021lora} method to fine-tune the [``$gate\_proj$'', ``$up\_proj$'', ``$down\_proj$''] modules in the feed-forward network of the Llama-2-7B model. We conducted all training on an NVIDIA 4090$\times$8 GPU. 
The hyper-parameters utilized for training KaLM (based on Llama-2-7B) are enumerated in Table~\ref{hyperparameters}.

\begin{table}[ht]
\centering
\caption{Hyper-parameters for training KaLM.}
\scalebox{0.80}{\begin{tabular}{ccc}
\toprule
Hyper-parameters              & WN18RR & FB15k-237 \\
\midrule
epochs                        & 20     & 15        \\
max-description-length         & 50     & 50        \\
max-language-modeling-length  & 256    & 256       \\
explicit-alignment-batch-size & 24     & 40        \\
implicit-alignment-batch-size & 4      & 6         \\
lora-module                   & ffn    & ffn       \\
lora-alpha                    & 16.0   & 16.0      \\
lora-drouout                  & 0.05   & 0.05      \\
lora-rank                     & 8      & 8         \\
bnb-config                           & load-in-8bit   & load-in-8bit     \\
learning-rate                 & 1e-4   & 1e-4      \\
LR-sheduler-type              & cosine & cosine    \\
weight-decay                  & 0.001  & 0.001     \\
gradient-checkpointing        & True   & True      \\
optimizer                     & AdamW  & AdamW     \\
AdamW-beta1                   & 0.9    & 0.9       \\
AdamW-beta2                   & 0.999  & 0.999     \\
bf16                          & True   & True     \\
\bottomrule
\end{tabular}}
\label{hyperparameters}
\end{table}

We also implemented KaLM based on other LLMs to demonstrate the generalizability of our approach, including Llama-3-8B, Mistral-7B-v0.1, OPT-6.7B, Pythia-6.9B, and Pythia-2.8B. 
It is important to note that the feed-forward network layers in the Pythia model are named [``$dense\_h\_to\_4h$'', ``$dense\_4h\_to\_h$''], while in the OPT model they are named [``$fc1$'', ``$fc2$'']. This differs from the feed-forward network layers in the Llama and Mistral model series. 
The parameters used in these experiments are shown in Table~\ref{hyperparameters-more} (only the differing parameters are listed; the unlisted parameters remain consistent with Table~\ref{hyperparameters}).

\begin{table*}[ht]
\centering
\caption{Additional Hyper-parameters for training \textit{KaLM} with different LLMs.}
\scalebox{1.00}{\begin{tabular}{ccccc}
\toprule
Models              & epochs & explicit-batch-size & implicit-batch-size & bnb-config    \\
\midrule
Llama-3-8B-WN                     & 20     & 18 & 3  & load-in-8bit     \\
Llama-3-8B-FB                     & 15     & 36 & 5 & load-in-8bit      \\
Mistral-7B-v0.1-WN                     & 20     & 40 & 5 & load-in-4bit      \\
Mistral-7B-v0.1-FB                     & 15     & 72 & 8  & load-in-4bit     \\
OPT-6.7B-WN                     & 20     & 24 & 3   & load-in-8bit    \\
OPT-6.7B-FB                     & 15     & 40 & 6   & load-in-8bit    \\
Pythia-6.9B-WN                     & 20     & 24 & 4  & load-in-8bit     \\
Pythia-6.9B-FB                     & 15     & 42 & 6   & load-in-8bit    \\
Pythia-2.8B-WN                     & 20     & 48 & 8   & load-in-8bit    \\
Pythia-2.8B-FB                     & 15     & 96 & 10  & load-in-8bit     \\
\bottomrule
\end{tabular}}
\label{hyperparameters-more}
\end{table*}

For the cosine similarity matrix composed of head entity-relation embeddings (row direction) and tail entity embeddings (column direction), we calculate the cross-entropy loss in the row direction (i.e., a head entity-relation embedding matching different tail entity embeddings) and the column direction (i.e., a tail entity embedding matching different head entity-relation embeddings) separately. We then take the average of the two losses to obtain the final InfoNCE loss. Similar to Equation~\ref{eq-row-cl}, the column-direction loss is defined as follows: 
\begin{align*}
    \label{eq-row-cl-more}
    \ell_c = - \log \frac{e^{(\phi(e_{t}, e_{hr})-\gamma)/\tau}}{e^{(\phi(e_{t}, e_{hr})-\gamma)/\tau} + \sum\nolimits_{j=1}^{\mathcal{N}} e^{\phi(e_{t}, e_{{hr}_{j}^{\prime}})/\tau}}.
\end{align*}

\subsection{More Details about Evaluations}

For the embedding-based KGC task, we report five automated metrics: Mean Rank (MR), Mean Reciprocal Rank (MRR), and Hit@$k$ ($k \in \{1, 3, 10\}$). 
MR is the mean rank of all test triplets and MRR denotes the average reciprocal rank of all test triples. Hit@$k$ measures the proportion of entities correctly ranked in the top $k$. 
Following previous work, our method is evaluated under the \textit{filtering setting} \citep{bordes2013translating}, where the scores of all true triples in the training, validation, and testing set are ignored. All results are averaged over the tail direction (a <head entity-relation> embedding matching different tail entity embeddings, i.e., tail entity prediction) and head direction (a <tail entity-inverse relation> embedding matching different head entity embeddings, i.e., head entity prediction). 

For the generation-based KGQA task, we report the prediction accuracy over head entities, tail entities,  relations, and relation classifications. 
To better prompt LLMs for the knowledge graph question-answering task, we selected several triples from the validation set and constructed few-shot examples using the corresponding templates from Table~\ref{kgqa-cases}.

\section{Addition Experimental Results}
\label{sec:appendix-results}

In this section, we provide more experimental results to show the effectiveness of our method.

\subsection{More Experiments on Knowledge Representation Assessment}
In Table~\ref{kgc-results-more}, we present additional knowledge representation results (the embedding-based KGC task) to demonstrate the effectiveness of KaLM in knowledge alignment. 
The best and second-best experimental results are indicated by \textbf{bold} and \underline{underline} texts, respectively. 
Overall, the proposed method achieved excellent performance on the embedding-based KGC task, delivering impressive results in the MR and Hit@10 metrics, while also being highly competitive in other metrics.

The experimental results based on LLMs of different sources and scales demonstrate the effectiveness and generalizability of our proposed method. 
Under similar experimental settings, more powerful LLMs (such as Llama3-8B and Mistral-7B) achieved better metrics after being fine-tuned with KaLM, which also demonstrates the scalability of our method. 
It is worth noting that for LLMs of the same origin but different scales (Pythia-6.9B and Pythia-2.8B), the smaller-scale Pythia-2.8B benefited from a larger training batch size during fine-tuning. As a result, its final experimental metrics matched or even surpassed those of the more powerful Pythia-6.9B model. 
This also highlights the importance of large batch sizes for the embedding-based KGC task, suggesting that using more powerful computing resources and larger GPU memory could further enhance the effectiveness of the proposed KaLM method.

\begin{table*}[t]
\centering
\caption{More Embedding-based KGC results with various LLMs on WN18RR and FB15k-237.}
{\begin{tabular}{l|ccccc|ccccc}
\hline
\multirow{2}{*}{\bf Method} & \multicolumn{5}{c|}{\bf WN18RR}      & \multicolumn{5}{c}{\bf FB15k-237}  \\ \cline{2-11} & \bf MR & \bf MRR & \bf H@1 & \bf H@3 & \bf H@10 & \bf MR & \bf MRR & \bf H@1 & \bf H@3 & \bf H@10 \\ \hline
\multicolumn{11}{l}{\textit{structure-based methods}}         \\ \hline
TransE & 2300 & 0.243 & 0.043  & 0.441 & 0.532 & 323 & 0.279 & 0.198 & 0.376 & 0.441 \\
DistMult & 7000 & 0.444 & 0.412 & 0.470 & 0.504  & 512 &  0.281 & 0.199 & 0.301 & 0.446 \\
RotatE & 3340 & 0.476 & 0.428 & 0.492 & 0.571 & 177 & 0.338 & 0.241 & 0.375 & 0.533 \\ \hline
\multicolumn{11}{l}{\textit{description-based methods (autoencoder PLMs)}}         \\ \hline
KG-BERT & 97 & 0.216 & 0.041 & 0.302 & 0.524 & 153 & - & - & - & 0.420 \\ 
StAR & 51 & 0.401 & 0.243 & 0.491 & 0.709 & 117 & 0.296 & 0.205 & 0.322 & 0.482 \\ 
C-LMKE & 72 & 0.598 & 0.480 & 0.675 & 0.806 & 183 & \textbf{0.404} & \textbf{0.324} & \textbf{0.439} & \textbf{0.556} \\
SimKGC & - & \textbf{0.671} & \textbf{0.587} & \textbf{0.731} & 0.817 & - & \underline{0.333} & \underline{0.246} & \underline{0.362} & 0.510 \\ \hline
\multicolumn{11}{l}{\textit{description-based methods (autoregressive LLMs)}}         \\ \hline
Llama-2-7B  & 15969 & 0.010 & 0.004 & 0.010 &0.020 & 5359 & 0.006 & 0.002 & 0.004  & 0.012 \\
\textbf{Llama2-7B$_{KaLM}$}  & \textbf{19} & 0.556 & 0.409 & 0.656 & 0.851 & \textbf{114} & 0.299 & 0.204 & 0.325  & 0.502 \\ 
\textbf{Llama3-8B$_{KaLM}$}  & 23 & 0.588 & 0.446 & 0.676 & \underline{0.860} & 121 & 0.308 & 0.212 & 0.337  & 0.509 \\ 
\textbf{Mistral-7B$_{KaLM}$}  & \underline{20} & \underline{0.612} & \underline{0.484} & \underline{0.702} & \textbf{0.869} & \underline{116} & 0.317 & 0.225 & 0.351  & \underline{0.518} \\ 
\textbf{OPT-6.7B$_{KaLM}$}  & 24 & 0.514 & 0.397 & 0.603 & 0.822 & 126 & 0.288 & 0.199 & 0.312  & 0.486 \\ 
\textbf{Pythia-6.9B$_{KaLM}$}  & 28 & 0.508 & 0.394 & 0.598 & 0.818 & 130 & 0.289 & 0.199 & 0.310  & 0.484 \\ 
\textbf{Pythia-2.8B$_{KaLM}$}  & 30 & 0.539 & 0.398 & 0.644 & 0.829 & 133 & 0.292 & 0.205 & 0.318  & 0.489 \\ 
\hline
\end{tabular}}
\label{kgc-results-more}
\end{table*}

\subsection{More Experiments on Knowledge Inference Evaluation}
In Figure~\ref{kgqa-results-more}, we present additional knowledge inference results (generation-based KGQA) to demonstrate the effectiveness of KaLM in knowledge alignment. 
This section demonstrates the performance of various powerful LLMs (including Llama-2-7B, Llama-3-8B, and Mistral-7B) before and after fine-tuning with KaLM, across various knowledge graph question-answering tasks (including head entity prediction, tail entity prediction, relation prediction, and triple classification). 

The experimental results can be divided into three groups by color: the green series, blue series, and red series correspond to the KGQA results of Llama-2-7B, Llama-3-8B, and Mistral-7B before and after training, respectively. 
It can be observed that after fine-tuning with KaLM, all three LLMs achieved consistent improvements in prediction accuracy for the question-answering tasks. 

At the KGQA task level, the most significant overall improvements were observed in tail entity prediction (an average increase of 14.1\%) and triple classification (an average increase of 12.7\%), followed by relation prediction (an average increase of 8.6\%) and head entity prediction (an average increase of 6.9\%). At the LLM level, the most exciting improvements were seen in Llama-3-8B (an average increase of 11.1\%) and Mistral-7B (an average increase of 10.8\%), while Llama-2-7B showed relatively smaller gains (an average increase of 9.6\%). This suggests that our method demonstrates better scalability with more powerful LLMs.

\begin{figure*}[ht]
\centering 
\includegraphics[width=\linewidth]{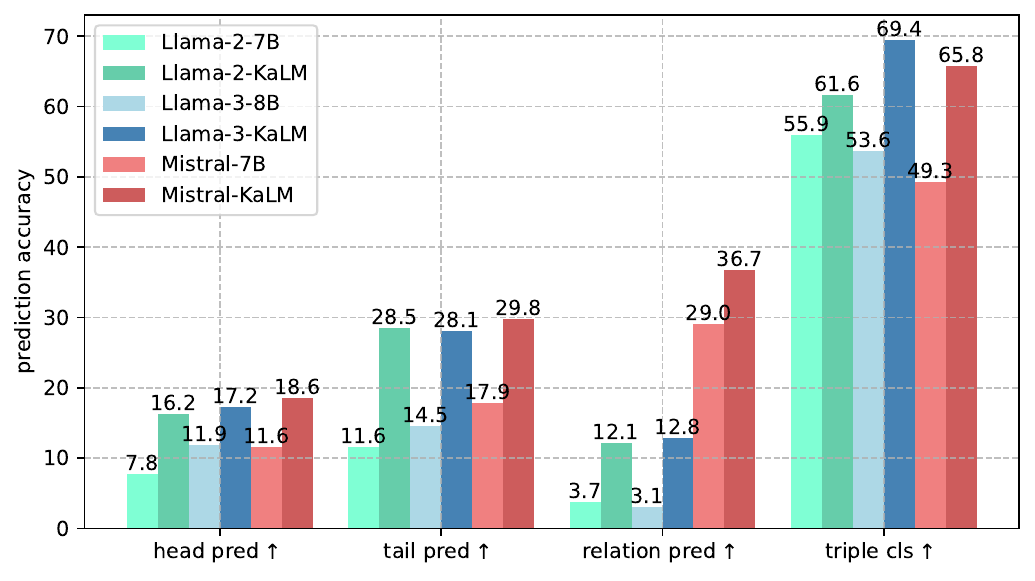}
\caption{Comparison of generative knowledge inference performance between Base LLMs and their fine-tuned KaLM versions, best viewed in three color groups. The symbol $\uparrow$ means higher is better and $\downarrow$ means lower is better.}
\label{kgqa-results-more}
\end{figure*}

\subsection{More Visualizations on Knowledge Representation Matrix}
From this section onward, unless stated otherwise, KaLM refers to the model checkpoint trained on Llama-2-7B using our method. 
We present more knowledge representation results to demonstrate the effectiveness of KaLM in knowledge alignment. 
Figure~\ref{fig-anisotropy-two} displays the sentence similarity matrix of several similar entity descriptions from the WN8RR dataset. 
Detailed information about entity names and descriptions can be found in Figure~\ref{ent-desc}. 
It is evident that KaLM can obtain more distinguishable knowledge representations, where the similarity between related entities (diagonal elements) is high, while the similarity between unrelated entities (off-diagonal elements) is low. 

\begin{figure*}[t]
\centering
\subfigure[LLaMA]{
    \begin{minipage}[b]{0.45\textwidth}
    \includegraphics[width=1\textwidth]{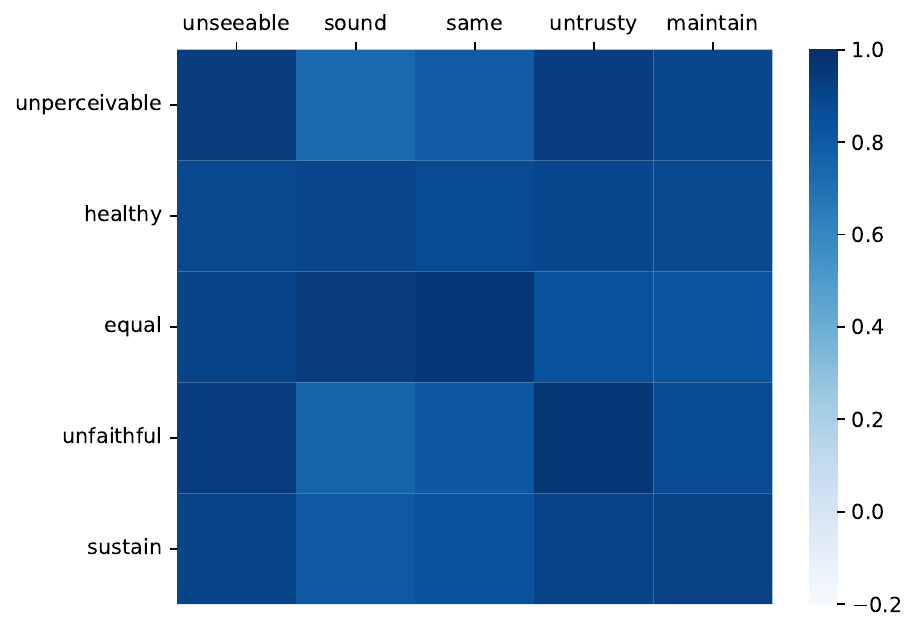}
    \end{minipage}
    \label{fig-anisotropy-two-base}
}
\quad
\subfigure[KaLM]{
    \begin{minipage}[b]{0.45\textwidth}
    \includegraphics[width=1\textwidth]{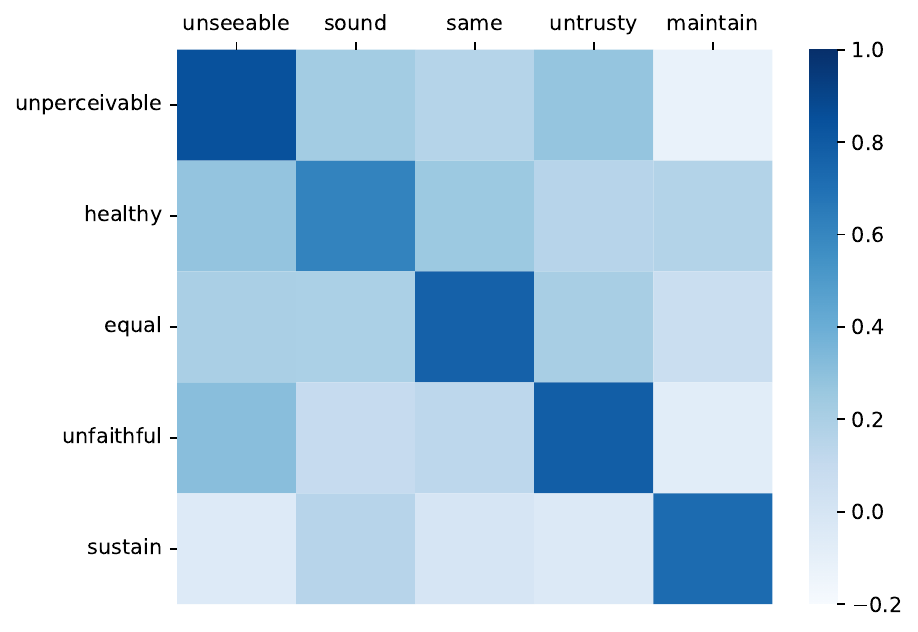}
    \end{minipage}
    \label{fig-anisotropy-two-kalm}
}
\caption{Similarity matrix of selected similar entity descriptions from the WN8RR dataset.}
\label{fig-anisotropy-two}
\end{figure*}

\begin{figure*}[ht]
\centering 
\includegraphics[width=\textwidth]{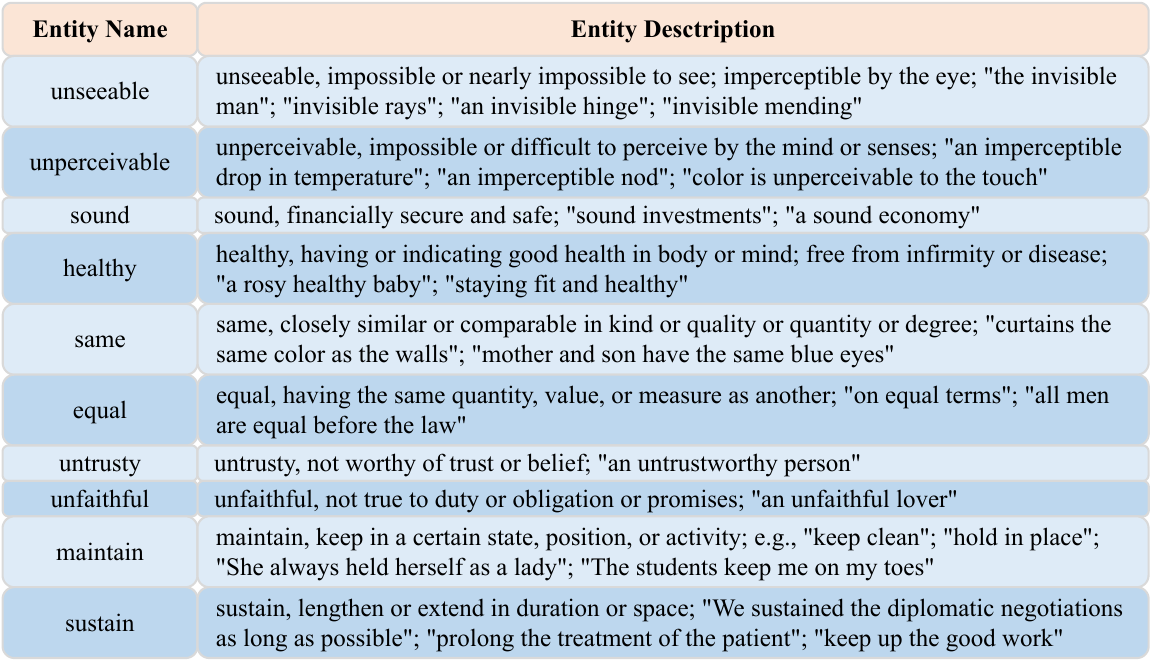}
\caption{Selected entities and their corresponding textual descriptions.}
\label{ent-desc}
\end{figure*}

\subsection{Detailed analysis of Representation Anisotropy}
We further analyze the sentence-level representation anisotropy on the Wikitext-103 test set using model checkpoints trained on the WN18RR dataset. The sentence-level anisotropy value for a given corpus $\{\mathcal{D}_i\}_{i=1}^N$ is defined in Equation~\ref{eq-anisotropy}, where a lower anisotropy value indicates better discriminative characteristics of sentence representations. 

\begin{figure*}[ht]
\begin{minipage}[t]{0.5\linewidth}
    \centering
    \includegraphics[width=1\textwidth]{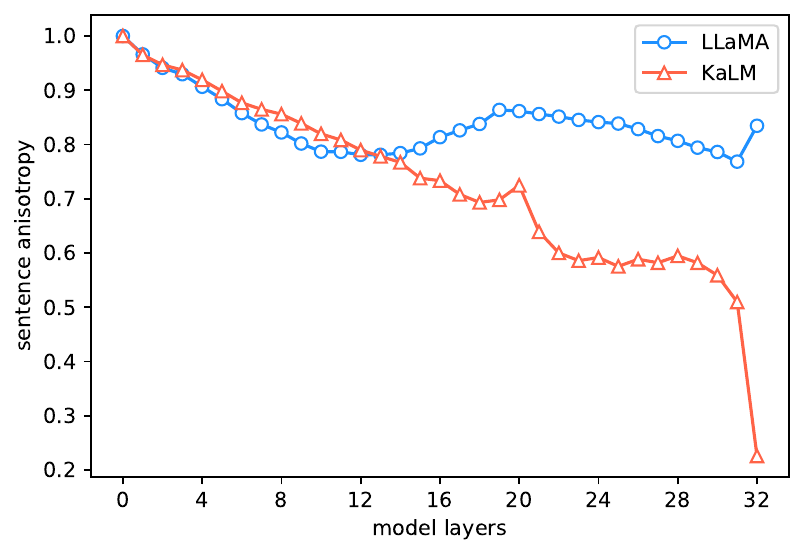}
    \caption{layer-wise analysis of anisotropy. The vertical axis represents the sentence-level representation anisotropy value on the Wikitext-103 test set, while the horizontal axis denotes the number of model layers.}
    \label{fig-anisotropy-chart-layer}
\end{minipage}
\quad
\begin{minipage}[t]{0.5\linewidth}
    \centering
    \includegraphics[width=1\textwidth]{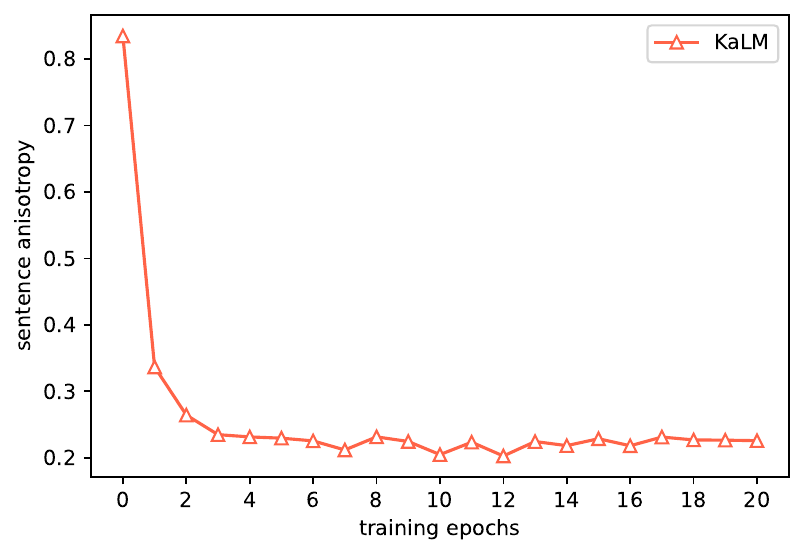}
    \caption{epoch-wise analysis of anisotropy. The vertical axis represents the sentence-level representation anisotropy value on the Wikitext-103 test set, while the horizontal axis denotes the number of training epochs.}
    \label{fig-anisotropy-chart-epoch}
\end{minipage}
\end{figure*}

Figure~\ref{fig-anisotropy-chart-layer} plots the anisotropy value over different layers for LLaMA and KaLM. We can observe that the anisotropy value of LLaMA consistently remains at a relatively high level, suggesting that the base LLM suffers from severe representation anisotropy issues. In contrast, our proposed KaLM notably mitigates this issue, with the anisotropy values decreasing gradually as the depth of the model increases, and dropping significantly from 0.5 to 0.2 at the output layer. 
The anisotropy values of the last layer for LLaMA and KaLM show that after training with our method, the sentence-level anisotropy value significantly decreased from 0.83 to 0.21. 
The results indicate that our method can effectively reduce the anisotropy of representations across layers in LLMs, resulting in a significant improvement in knowledge representation. 

Figure~\ref{fig-anisotropy-chart-epoch} analyzes the changes in anisotropy values during the model training process. The results show that the anisotropy values decrease rapidly after a few epochs of training and eventually stabilize at a low level. 
We assume that the initial epochs of training have completed the preliminary alignment of knowledge representation, while the subsequent training epochs mainly focus on integrating explicit and implicit representations. 

\section{Ablation Studies}
\label{sec:ablations}
In this section, we present concrete ablation studies to analyze the effectiveness of each component of our approach. We ablate the settings that led to the final design, including training objectives, fine-tuning modules, and training epochs. 
It is important to note that the results of the ablation experiments in this section were obtained from earlier runs on an NVIDIA 3090$\times$4 GPU, which may lead to slight differences compared to the full KGC results presented in the main text.

\subsection{The necessity of the implicit knowledge alignment objective (Equation~\ref{eq:l_pat})}
In Table~\ref{kgc-ablation-lambda}, we train the model using different loss weights (i.e., the $\lambda$ parameter in Equation~\ref{eq:final-obj}) and analyze its performance on the KGC task. Note that this experiment is conducted solely for ablation analysis, thus only 10 training epochs are used. 
Experimental results reveal that incorporating the implicit knowledge alignment objective (i.e., $\lambda > 0$) generally leads to better performance in KGC, indicating further improvement in knowledge representation. 
The best performance in KGC is achieved when $\lambda=0.1$. 
The results confirm that both explicit alignment and implicit alignment are crucial for knowledge alignment, as they both essentially require a deep understanding of knowledge. 

The implicit knowledge alignment objective focuses on incorporating textual patterns of knowledge into the LLM to prevent catastrophic forgetting of previous knowledge and maintain its generative capability. 
We also conducted additional perplexity (PPL) evaluation experiments to illustrate the impact of the implicit knowledge alignment loss. The additional results show that for the corresponding $\lambda = 0, 0.01, 0.1, 1.0$ in Table~\ref{kgc-ablation-lambda}, the model's PPL are \textbf{6.42}, \textit{4.96}, \textit{4.97}, and \textit{4.98}, respectively. Therefore, we can conclude that incorporating the implicit alignment loss maintains the model's language modeling capability, whereas not using the implicit alignment loss significantly impairs the model's generative ability.

\begin{table}[ht]
\centering
\caption{KGC results with different $\lambda$ in Equation~\ref{eq:final-obj}.}
\scalebox{0.75}{\begin{tabular}{l|ccccc|c}
\hline
\multirow{2}{*}{\bf Method} & \multicolumn{5}{c|}{\bf WN18RR} & \multirow{2}{*}{\bf PPL}    \\
\cline{2-6} & \bf MR & \bf MRR & \bf H@1 & \bf H@3 & \bf H@10 \\ \hline
KaLM ($\lambda=0$)  & 21.2 & 0.512 & 0.355 & 0.611 & 0.815 & 6.42\\ \hline
KaLM ($\lambda=0.01$)  & \textbf{19.8} & 0.510 & 0.352 & 0.604 & 0.818 & 4.96\\ \hline
KaLM ($\lambda=0.1$)  & 20.1 & \textbf{0.517} & \textbf{0.359} & \textbf{0.615} & \textbf{0.825} & 4.98\\ \hline
KaLM ($\lambda=1.0$)  & 21.6 & 0.500 & 0.336 & 0.596 & 0.806 & 4.98\\ \hline
\end{tabular}}
\label{kgc-ablation-lambda}
\end{table}

\subsection{The effects of fine-tuning different LLM modules using LoRA}
In Table~\ref{kgc-ablation-lora}, we fine-tune different modules of the model using the LoRA \cite{hu2021lora} method and analyze their performance on KGC tasks and PPL evaluations. Note that this experiment is conducted solely for ablation analysis, hence only 10 epochs of training were performed. 
``\textit{att}'' indicates fine-tuning only the attention module, ``\textit{ffn}'' indicates fine-tuning only the feed-forward network, and ``\textit{att-ffn}'' indicates fine-tuning both the attention module and the feed-forward network simultaneously. 
The results show that fine-tuning with the ``\textit{att-ffn}'' approach achieves the best KGC performance, but it also leads to higher PPL values, suggesting that the model's generation capability may be significantly compromised. 
Therefore, as a compromise, we choose the ``\textit{ffn}'' fine-tuning approach, maintaining moderate knowledge representation performance while preserving the original generation capability. 

These experimental results are consistent with the conclusions of \cite{he2021towards}, where the FFN learns local features and patterns within the input sequence, allowing it to directly capture task-specific text patterns. Meanwhile, attention provides the model with the ability to capture complex contextual relationships, which is key to LLMs' understanding and generation of natural language. 
Under the knowledge-aligned language modeling objective, we aim to align the internal knowledge representations of LLMs while preserving their inherent natural language generation capabilities. Therefore, directly fine-tuning the FFN layers can reduce resource consumption and maximize the effectiveness of KaLM fine-tuning.

\begin{table}[ht]
\centering
\caption{KGC results and PPL evaluation results when fine-tuning different network modules with LoRA.}
\scalebox{0.75}{\begin{tabular}{l|ccccc|c}
\hline
\multirow{2}{*}{\bf Method} & \multicolumn{5}{c|}{\bf WN18RR} & \multirow{2}{*}{\bf PPL}    \\
\cline{2-6} & \bf MR & \bf MRR & \bf H@1 & \bf H@3 & \bf H@10 \\ \hline
KaLM (att)  & 21.9 & 0.47.5 & 0.331 & 0.580 & 0.784 & 5.03\\ \hline
KaLM (ffn)  & 20.1 & 0.517 & 0.359 & 0.615 & 0.825 & 4.96\\ \hline
KaLM (att-ffn)  & \textbf{19.5} & \textbf{0.525} & \textbf{0.371} & \textbf{0.619} & \textbf{0.831} & 5.07\\ \hline
\end{tabular}}
\label{kgc-ablation-lora}
\end{table}

\subsection{The sustained gains and potential impacts of training for more epochs}
In Table~\ref{kgc-ablation-epoch}, we fine-tune the model using different numbers of training epochs and analyze their performance on KGC tasks. This experiment is mainly conducted to investigate whether additional training epochs can lead to further improvement in knowledge representations. 
The experimental results show that using more training epochs can continuously improve the performance of KaLM on the KGC task, resulting in higher MRR and Hit@k metrics. 
The model trained with our method consistently maintains an acceptable PPL value due to the implicit knowledge alignment objective. However, this also comes with more computational resource consumption and training time. As a result, we selected a moderate number of training epochs. 

\begin{table}[ht]
\centering
\caption{KGC results with different training epochs.}
\scalebox{0.75}{\begin{tabular}{l|ccccc|c}
\hline
\multirow{2}{*}{\bf Method} & \multicolumn{5}{c|}{\bf WN18RR} & \multirow{2}{*}{\bf PPL}    \\
\cline{2-6} & \bf MR & \bf MRR & \bf H@1 & \bf H@3 & \bf H@10 \\ \hline
KaLM (epoch=10)  & 20.1 & 0.517 & 0.359 & 0.615 & 0.825 & 4.96\\ \hline
KaLM (epoch=20)  & \textbf{19.6} & 0.554 & 0.402 & 0.650 & 0.848 & 4.98\\ \hline
KaLM (epoch=30)  & 21.9 & \textbf{0.576} & \textbf{0.427} & \textbf{0.673} & \textbf{0.854} & 5.00\\ \hline
\end{tabular}}
\label{kgc-ablation-epoch}
\end{table}

\end{document}